\documentclass{informs4}
\RequirePackage{tgtermes}
\RequirePackage{newtxtext}
\RequirePackage{newtxmath}
\RequirePackage{bm}
\RequirePackage{endnotes}

\OneAndAHalfSpacedXII

\usepackage{mathtools}
\usepackage{algorithm}
\usepackage{algpseudocode}
\usepackage{booktabs}
\usepackage{tikz}
\usetikzlibrary{positioning,arrows.meta}
\usepackage{enumitem}

\usepackage{natbib}
 \bibpunct[, ]{(}{)}{,}{a}{}{,}%
 \def\bibfont{\small}%
\EquationsNumberedThrough
\TheoremsNumberedThrough
\ECRepeatTheorems

\MANUSCRIPTNO{}

\newcommand{\E}{\mathbb{E}}

\newcommand{\Prob}{\mathbb{P}}

\newcommand{\Attn}{\mathrm{Attn}}

\makeatletter
\long\def\theARTICLETOPLEFT{}%
\def\theARTICLETOPRIGHT{}%
\def\setfirstoddFOOT{}%
\def\setfirstevenFOOT{}%
\def\ps@headings{%
  \let\@oddfoot\@empty\let\@evenfoot\@empty
  \def\@evenhead{\thepage\hfil}%
  \def\@oddhead{\hfil\thepage}%
}%
\long\def\ARTICLEAUTHORS#1{\gdef\theARTICLEAUTHORS{%
  \begin{center}\HOOKa\vspace*{0pt}#1\vspace*{0pt}\end{center}}}%
\def\AUTHOR#1{%
  \begin{center}
  \AUTHORfont\HD{15}{0}#1\HD{0}{6}\relax
  \end{center}}
\def\AFF#1{%
  \begin{center}
  \AFFfont{#1}\relax
  \vskip1.6pt
  \end{center}}
\def\theARTICLETITLE{%
  \HOOKtop
  \begin{center}
  \vspace*{0pt}%
  \TITLEfont\HD{24}{0}\sffamily\bfseries\theTITLE\HD{0}{15}%
  \end{center}}
\makeatother

\begin{document}

\RUNAUTHOR{Emadi}
\RUNTITLE{The Critical Horizon}
\TITLE{The Critical Horizon: Inspection Design Principles for
  Multi-Stage Operations and Deep Reasoning}
\ARTICLEAUTHORS{%
\AUTHOR{Seyed Emadi}
\AFF{Kenan-Flagler Business School, University of North Carolina at Chapel Hill,
  \EMAIL{seyed\_emadi@kenan-flagler.unc.edu}}
}
\ABSTRACT{%
Manufacturing lines, service journeys, supply chains, and AI reasoning
chains share a common challenge: attributing a terminal outcome to the
intermediate stage that caused it.  We establish an information-theoretic
barrier to this \emph{credit assignment problem}: the signal connecting
early steps to final outcomes decays exponentially with depth, creating
a \emph{critical horizon} beyond which reliable learning from
endpoint data alone requires exponentially many samples.  We prove four results.  First, a Signal Decay
Bound: sample complexity for attributing outcomes to early stages grows
exponentially in the number of intervening steps.  Second, Width Limits:
parallel rollouts provide only logarithmic relief, with correlation
capping the effective number of independent samples.  Third, an
Objective Mismatch: additive reward aggregation optimizes the wrong
quantity when sequential validity requires all steps to be correct.
Fourth, Optimal Inspection Design: uniform checkpoint spacing is
minimax-optimal under homogeneous signal attenuation, while a greedy
algorithm yields optimal non-uniform schedules under heterogeneous
attenuation.  Together, these results provide a common analytical
foundation for inspection design in operations and supervision design
in AI.
}%


\maketitle


\section{Introduction}
\label{sec:intro}

When a semiconductor wafer fails final testing, which of the hundred preceding fabrication steps caused the defect? When a customer churns after a multi-touchpoint journey, which interaction was responsible? When a multi-tier supply chain delivers a defective product, which supplier introduced the flaw? These are instances of the \emph{credit assignment problem}: attributing terminal outcomes to intermediate stages in deep sequential processes.

Classical inspection theory \citep{lindsay1964optimal, raz1986economic, tapiero1987quality} addresses this challenge through intermediate checkpoints, optimizing their placement to balance inspection cost against defect-escape risk. But intermediate inspection is not always an option: a proof step that appears irregular may be valid in context; a component may only be testable once assembled; a customer's reaction to one touchpoint depends on what follows. In such cases, validity is a property of the whole sequence, and terminal inspection is not merely cheaper but necessary. Moreover, even when terminal inspection is the only option, the classical literature implicitly assumes that defect signals persist through the production chain: with enough terminal observations, any stage's quality can eventually be inferred. We show this assumption fails for deep processes, where signals decay exponentially with depth.

The challenge of attributing outcomes to early stages under terminal-only observation has gained particular urgency in AI, where large language models generate extended reasoning chains but only the final answer can be verified. Empirically, reinforcement learning (RL) value networks trained on outcome rewards ``collapse'': they converge to near-constant predictions for early steps, as if outcomes carry no information about the chain's beginning \citep{vineppo2024}. Process reward models that provide step-level supervision dramatically improve performance \citep{lightman2023lets}, but require expensive human annotation. When is such annotation necessary, and when does outcome supervision suffice?

This paper answers these questions by establishing information-theoretic foundations for credit assignment. Using Strong Data Processing Inequalities (SDPIs), we prove that signals connecting early steps to terminal outcomes decay exponentially with depth. The analysis reveals a sharp threshold: there exists a \emph{critical horizon} $H_{\mathrm{crit}}$ such that steps within this distance of the outcome can be evaluated from endpoint data, while steps beyond require a number of samples that grows exponentially in the gap $H-t$. This is not an algorithmic limitation but an information-theoretic barrier. The critical horizon scales logarithmically with sample size and signal strength but inversely with a per-stage contraction rate that captures how much each processing step blurs the distinction between good and bad earlier steps. In AI reasoning systems, sampling temperature controls this blurring: higher temperature produces more diverse continuations, making it harder to trace outcomes back to early steps. In manufacturing, the analogous driver is the overlap among downstream quality-grade distributions: variability-reducing steps (polishing, annealing, rework), pooling operations (lot mixing, kitting, assembly), and coarse metrology bins all cause products of different upstream grades to produce similar distributions of grades at the next stage, fading the signal about the original grade. The critical horizon formula thus determines whether intermediate feedback is a cost-saving measure or a \emph{feasibility requirement}.

A key insight underlying our analysis is that credit assignment targets \emph{abstract states}, not raw observations. In manufacturing, the observable state at production stage $t$ consists of detailed sensor readings (temperatures, pressures, vibration spectra), but credit assignment asks whether the stage produced output that was within specification, marginal, or defective. Multiple sensor patterns correspond to the same quality grade. In AI reasoning, the observable state is the token sequence generated after $t$ reasoning steps, but credit assignment asks whether the logical content is valid, not which specific words were used. The phrases ``Therefore $x=5$'' and ``Thus we conclude $x$ equals $5$'' use different tokens but express the same mathematical content. In both settings, the space of abstract states (quality grades, semantic content) is smaller than the space of raw observations (sensor readings, token sequences).

This many-to-one mapping from observations to abstract states is the mechanism behind the per-stage contraction introduced above: when multiple observable states collapse to the same abstract state, transitions from distinct starting points lead to overlapping distributions over next states, and statistical distinguishability shrinks. This occurs even when raw dynamics are deterministic (as with token generation in language models), because the induced dynamics on the abstract space contract. The implication for learning systems is immediate: when signal decay is severe, outcome distributions conditioned on different early states become nearly identical, and training data contains almost no information about which early states are better. Under the assumptions of our model, the critic collapse phenomenon observed in AI systems \citep{vineppo2024} is \emph{statistically optimal behavior} given the available information, and not a bug to be fixed but a symptom of a fundamental limit.

A natural response to signal decay is to generate more observations: multiple parallel trajectories from the same starting point, as in replicated testing or width-based methods like GRPO (Group Relative Policy Optimization; \citealp{deepseek2024}). We show this strategy provides only limited relief: correlation among rollouts caps the effective width at $1/\rho$, where $\rho$ is the pairwise correlation between rollout outcomes, and adding more rollouts yields diminishing returns.

For steps beyond the critical horizon, process supervision becomes necessary, raising a new question: how should step-level feedback be aggregated? Sequential validity requires \emph{all} steps to be correct, a multiplicative rather than additive criterion. The mismatch is stark: a 100-step process with 99\% per-step accuracy has an expected correctness score of 99, yet the probability that all 100 steps are correct is only $0.99^{100} \approx 0.37$. Systems that optimize additive objectives learn to produce chains that are ``mostly correct'' but nonetheless invalid. We analyze this objective mismatch and propose curriculum strategies that transition from learnable additive objectives to correct multiplicative ones as policy quality improves.

Once the need for process supervision is established, a natural question arises: where should intermediate checkpoints be placed? Each checkpoint resets the signal decay clock, so a step immediately before a checkpoint is easy to evaluate, while a step far from any checkpoint is hard. The hardest step to evaluate under any schedule is the one farthest from its next checkpoint, and the design goal is to make this worst case as mild as possible. We derive inspection schedules that are minimax-optimal in this sense: they minimize the worst-case sample complexity across all steps. When all stages attenuate signal at the same rate, uniform checkpoint spacing is minimax-optimal. When different stages attenuate at different rates, a greedy algorithm that places each successive checkpoint as far as the information budget allows yields optimal non-uniform schedules, automatically concentrating inspections where signal loss is fastest. In either case, achieving polynomial sample complexity requires inspection density that scales nearly linearly with process depth: deeper processes demand proportionally more intermediate verification.

We develop a detailed application to large language model reasoning (Section~\ref{sec:llm-application} in the Electronic Companion), where the state abstraction maps token sequences to semantic content and sampling temperature controls the contraction coefficient. We validate the theoretical predictions through numerical experiments on both synthetic Markov chains and LLM reasoning chains on GSM8K (a grade-school math benchmark); full experimental details appear in Appendix~\ref{app:numerical-details}.

\paragraph{Summary of contributions.}
We establish four main results that transform the intuition ``credit assignment in long processes is hard'' into precise, actionable principles.

\begin{itemize}
\item \textbf{Signal Decay Bound} (Section~\ref{sec:distinguish}): Information about any given step decays exponentially when propagating to the terminal outcome. The resulting sample complexity bounds are tight, establishing that exponential scaling is fundamental. This yields a closed-form critical horizon formula.

\item \textbf{Width Limits} (Section~\ref{sec:width}): Parallel rollouts reduce estimator variance, but correlation caps effective width at $1/\rho$, extending the critical horizon only logarithmically.

\item \textbf{Objective Mismatch} (Section~\ref{sec:bellman}): Additive reward aggregation optimizes the wrong quantity when sequential validity requires all steps to be correct. We propose curriculum strategies that interpolate between learnable and correct objectives.

\item \textbf{Optimal Inspection Design} (Section~\ref{sec:inspection-design}): Uniform checkpoint spacing is minimax-optimal under homogeneous contraction; a greedy algorithm yields optimal non-uniform schedules under heterogeneous contraction. Polynomial sample complexity requires inspection density scaling as $\Omega(H / \log H)$.
\end{itemize}

The framework yields concrete guidance for both AI and operations practitioners. For AI systems, it explains why process reward models improve performance (they provide signal beyond the critical horizon), why methods like VinePPO (value-based intermediate-step process supervision; \citealp{vineppo2024}) succeed (they reduce effective depth below $H_{\mathrm{crit}}$), and why width-based methods face diminishing returns (correlation caps effective width). For operations managers, it provides a direct diagnostic. Consider a 50-stage semiconductor process. With contraction rate $\eta = 0.85$, signal strength $\Delta^2 = 0.2$, and 10,000 test wafers, $H_{\mathrm{crit}} \approx 47$ (simplified; $\approx 48$ with $\epsilon=0.1$): at least one intermediate checkpoint is required. With $\eta = 0.90$, the critical horizon extends to 72 steps, and terminal testing alone suffices. (In practice, $\eta$ is auditable: estimate adjacent-stage grade transition matrices from production data and compute the overlap-based bounds in Section~\ref{subsec:contraction}.)

The unifying insight is that credit assignment difficulty is a structural barrier, not an algorithmic one. The critical horizon formula answers the diagnostic question (\emph{when} is intermediate feedback necessary?); the optimal inspection schedules answer the prescriptive question (\emph{where} to place checkpoints?); and the objective mismatch analysis answers the aggregation question (\emph{how} to combine step-level feedback?). Together, these results provide a principled foundation for supervision system design in both operations and AI.

\subsection{Literature Review}

Our results sit at the intersection of inspection design in operations research, outcome-versus-process supervision in deep sequential decision systems, and information-theoretic limits for Markovian signal propagation.
On the OR side, we build on classical checkpoint placement models (\citealp{lindsay1964optimal, raz1986economic, tapiero1987quality}), but focus on a different binding constraint: \emph{observability} under terminal-only evaluation in deep processes, rather than expected-cost tradeoffs under persistent defect signals.
On the AI side, recent work documents that outcome-only supervision can fail on long reasoning chains and that step-level supervision can help (\citealp{lightman2023lets, uesato2022solving, vineppo2024}); we provide a non-asymptotic, information-theoretic characterization of \emph{when} outcome supervision becomes insufficient.
In standard RL settings with bounded concentrability, outcome and process supervision are polynomially equivalent \citep{jia2025verify}; our setting falls outside this regime because concentrability grows exponentially in horizon. Independent concurrent work by \citet{chen2025exponential} establishes an exponential separation using function approximation complexity, complementing our information-theoretic approach.

More broadly, this work contributes to a growing OR-for-AI agenda that applies operations research methodology to AI system design (\citealp{snell2024scaling, besta2025survey, elmachtoub2022smart, kulkarni2025synergizing}), here showing that classical tools such as strong data processing inequalities and minimax inspection design yield precise answers to questions the AI community has addressed primarily through empirical heuristics.
Technically, the analysis leverages strong data processing inequalities (\citealp{polyanskiy2017strong, raginsky2016strong}) and Le~Cam-style hypothesis testing lower bounds \citep{lecam1973convergence} to establish tight sample-complexity limits and the closed-form critical horizon.
A detailed, stream-by-stream literature discussion (inspection design, RL theory, information theory, and OR-for-AI) is provided in the Electronic Companion (Appendix~\ref{app:litreview}).

\subsection{Paper Organization}

The paper develops a logical progression from information-theoretic limits to optimal inspection design. Section~\ref{sec:prelim} establishes the formal framework: the MDP model, state abstraction, supervision regimes, and the contraction property that drives signal decay. Section~\ref{sec:distinguish} proves the Signal Decay Bound, showing that information about step $t$ decays as $\eta^{H-t}$, and derives the critical horizon formula that determines which steps can be evaluated from terminal data.

Two extensions follow. Section~\ref{sec:width} analyzes width amplification, showing that correlation among parallel rollouts caps effective width and extends the critical horizon only logarithmically. Section~\ref{sec:bellman} addresses an orthogonal failure mode: standard additive objectives optimize the wrong quantity for sequential validity, and proposes curriculum strategies that interpolate between learnable and correct objectives.

Section~\ref{sec:inspection-design} synthesizes these results into optimal inspection schedules: uniform spacing under homogeneous contraction and a greedy information-distance algorithm under heterogeneous contraction. Section~\ref{sec:numerical} provides a concise summary of numerical validation; full experimental setups and all plots are deferred to the Electronic Companion (Appendix~\ref{app:numerical-details}). The Electronic Companion also develops a detailed application to language model reasoning (Section~\ref{sec:llm-application}).

\section{Problem Formulation}
\label{sec:prelim}

We formalize the credit assignment problem in sequential processes, beginning with the operational setting before developing the mathematical framework.

\subsection{Multi-Stage Processes with Terminal Inspection}

Consider a process where items pass through $H$ stages before final inspection. At each stage, the item's state changes and defects may be introduced. The final inspection reveals whether the complete item meets specifications but does not identify which stage caused a failure.

As discussed in Section~\ref{sec:intro}, terminal-only inspection arises either because intermediate inspection is \emph{costly} (semiconductor metrology, supplier auditing) or \emph{infeasible} (validity depends on context that emerges only downstream). The common structure across manufacturing, service operations, supply chains, and AI reasoning is that terminal outcomes are observable but the intermediate stage responsible for success or failure is not. Without localization, process improvement efforts may target the wrong stages.

Our framework applies when (i) the process evolves through discrete stages with Markovian dynamics on an abstract state space, (ii) terminal inspection reveals aggregate quality, and (iii) intermediate inspection, if feasible, is costly. In AI, \emph{outcome supervision} refers to training on final output labels, while \emph{process supervision} provides step-level feedback. We adopt this terminology throughout, noting that the underlying structure is domain-general.

\subsection{Mathematical Formalization}

We model the sequential process as an episodic finite-horizon Markov decision process (MDP) with terminal reward. Let $\mathcal{S}$ denote a state space, $\mathcal{A}$ an action space, and $H$ the horizon (the number of decision stages).

\begin{definition}[Finite-Horizon MDP with Terminal Reward]
\label{def:sdp}
A \emph{finite-horizon MDP with terminal reward} consists of:
\begin{itemize}
    \item An initial state $S_0 \in \mathcal{S}$, drawn from a distribution $\mu$ over $\mathcal{S}$ (or fixed at a given state $s_0$)
    \item Transition dynamics, either deterministic $s_{t+1} = f(s_t, a_t)$ or stochastic $s_{t+1} \sim P(\cdot \mid s_t, a_t)$
    \item A policy $\pi = (\pi_0, \ldots, \pi_{H-1})$ where each $\pi_t: \mathcal{S} \to \Delta(\mathcal{A})$ maps states to distributions over actions
    \item A terminal evaluation $R = g(S_H) \in \{0, 1\}$ indicating success or failure
    \item Step validity indicators $r_t = h_t(S_t, A_t, S_{t+1}) \in \{0, 1\}$, treated as exogenous labels available only under process or partial supervision; the functions $h_t$ are problem-dependent and need not be known to the learner
\end{itemize}
A \emph{trajectory} is the complete sequence $\tau = (S_0, A_0, S_1, \ldots, A_{H-1}, S_H)$. We use uppercase letters for random variables; realized samples are denoted $\tau^{(i)} = (s_0^{(i)}, a_0^{(i)}, \ldots, s_H^{(i)})$. In {\em sequential validity} tasks, terminal success requires every step to be correct: $R = \prod_{t=0}^{H-1} r_t$; this structural relationship motivates the multiplicative objective introduced in Section~\ref{sec:bellman}.
\end{definition}

To ground this abstraction: in manufacturing, the state $S_t$ is the workpiece condition at stage $t$, an action $A_t$ is the process setting (temperature, pressure, feed rate), and the transition $f$ maps current condition and setting to the next-stage condition. The terminal evaluation $R$ indicates whether the finished product passes final inspection. In AI reasoning, $S_t$ is the partial solution generated so far, $A_t$ is the next reasoning step, and $R$ indicates whether the final answer is correct.

\begin{remark}[Deterministic Dynamics with Stochastic Trajectories]
\label{rem:dynamics}
In many sequential processes, the transition function $f$ is deterministic while trajectories are stochastic because actions are sampled from a policy distribution. This distinction matters: information loss arises not from noisy dynamics but from the many-to-one structure of state abstraction (Section~\ref{sec:distinguish}).
In manufacturing, this occurs even under fixed recipes: units sharing the same quality grade may differ in latent characteristics (material lot, tool drift, defect mode), so the \emph{abstraction} from detailed workpiece state to grade is many-to-one, and the induced grade-to-grade kernels have overlapping rows and $\eta<1$.
\end{remark}

The value function captures the probability of eventual success from any intermediate state. Intuitively, it answers the question: ``If the process is currently at state $s$ after $t$ steps, what is the chance it will ultimately succeed?''

\begin{definition}[Value Function]
\label{def:value}
The \emph{value function} under policy $\pi$ from state $s$ at step $t$ is
\[{
    V^\pi_t(s) := \Prob_\pi(R = 1 \mid S_t = s) = \E_\pi[R \mid S_t = s],
}\]

representing the probability of terminal success when starting from state $s$ at step $t$ and following policy $\pi$.
\end{definition}

\subsection{Supervision Models}

The learner observes $n$ i.i.d.\ trajectories generated by a fixed behavior policy $\pi$. We define three supervision regimes, distinguished by what feedback accompanies each trajectory.

\begin{definition}[Outcome Supervision]
\label{def:outcome}
Under \emph{outcome supervision}, the learner observes $n$ independent trajectory-outcome pairs:
\[{
    \mathcal{D}_{\mathrm{out}} = \{(\tau^{(i)}, R^{(i)})\}_{i=1}^n.
}\]

The full trajectory $\tau^{(i)}$ including all intermediate states is observed, but no step-level correctness labels are provided.
\end{definition}

The learner sees \emph{what} happened at each step but not \emph{whether} each step was correct. Consequently, the only supervision signal about step validity is through $R$, and any test for whether step $t$ was correct must rely on the induced separation in outcome distributions; this is why the lower bounds in Section~\ref{sec:distinguish} are stated in terms of $P^{(0)}_R$ versus $P^{(1)}_R$.

\begin{definition}[Process Supervision]
\label{def:process}
Under \emph{process supervision}, the learner additionally observes step validity labels:
\[{
    \mathcal{D}_{\mathrm{proc}} = \{(\tau^{(i)}, r_0^{(i)}, \ldots, r_{H-1}^{(i)}, R^{(i)})\}_{i=1}^n,
}\]

where $r_t^{(i)} \in \{0, 1\}$ indicates whether step $t$ of trajectory $i$ is valid.
\end{definition}

Process supervision provides the ground truth about each step. This is the gold standard for credit assignment but is expensive: a human expert must evaluate every step of every trajectory, or a formal verifier must check each logical inference.

\begin{definition}[Partial Supervision]
\label{def:partial}
Under \emph{$k$-step supervision}, the learner observes step labels at intervals of $k$ steps:
\[{
    \mathcal{D}_{k} = \{(\tau^{(i)}, r_k^{(i)}, r_{2k}^{(i)}, \ldots, r_{\lfloor (H-1)/k \rfloor \cdot k}^{(i)}, R^{(i)})\}_{i=1}^n,
}\]

That is, labels are provided at steps $t = k, 2k, \ldots$ (we omit $t = 0$ by convention). This interpolates between near-process supervision ($k = 1$, all steps except $r_0$ labeled) and outcome supervision ($k \geq H$, no step labels). The convention ensures that coarse supervision provides no step-level information.
\end{definition}

Under all three regimes, the learner observes the complete state sequence $(s_0, \ldots, s_H)$. The difficulty is not hidden states but missing correctness labels: the model's output is fully readable, but determining whether each step is valid may require human experts or formal verification.

\subsection{Information-Theoretic Preliminaries}

Our lower bounds rely on tracking how information degrades through sequential processing. The core question is: if two hypotheses (say, ``step $t$ was good'' versus ``step $t$ was bad'') lead to different distributions over trajectories, how many samples suffice to distinguish them?

The answer depends on how different the resulting outcome distributions are: nearly identical distributions require many samples, while well-separated distributions require few. We formalize this using $\chi^2$-divergence.

\begin{definition}[$\chi^2$-Divergence]
\label{def:chisq}
For distributions $P$ and $Q$ with $P$ absolutely continuous with respect to $Q$:
\[{
    \chi^2(P \| Q) := \int \frac{(dP - dQ)^2}{dQ} = \E_Q\left[\left(\frac{dP}{dQ} - 1\right)^2\right].
}\]

\end{definition}

We use $\chi^2$-divergence rather than KL divergence for two reasons. First, it connects directly to hypothesis testing: by Le~Cam's method \citep{lecam1973convergence}, distinguishing $P$ from $Q$ with constant probability requires $n = \Omega(1/\chi^2(P \| Q))$ samples. Second, it tensorizes cleanly across independent samples: $\chi^2(P^{\otimes n} \| Q^{\otimes n}) = (1 + \chi^2(P \| Q))^n - 1$ (see, e.g., \citealp{polyanskiy2017strong}). This multiplicative structure translates divergence bounds directly into sample complexity bounds.

The central tool in our analysis is the Strong Data Processing Inequality, which captures a fundamental phenomenon: \emph{processing information can only destroy it, never create it}. If two distributions are hard to distinguish, passing them through a noisy channel makes them even harder to distinguish.

\begin{definition}[$\chi^2$-Contraction Coefficient]
\label{def:contraction}
For a Markov kernel $K: \mathcal{X} \to \Delta(\mathcal{Y})$, the \emph{$\chi^2$-contraction coefficient} is
\[{
    \eta_{\chi^2}(K) := \sup_{\substack{P, Q: P \ll Q \\ 0 < \chi^2(P \| Q) < \infty}} \frac{\chi^2(PK \| QK)}{\chi^2(P \| Q)},
}\]

where $PK(y) = \int K(y \mid x) \, dP(x)$ is the pushforward measure.
\end{definition}

The contraction coefficient satisfies $\eta_{\chi^2}(K) \in [0, 1]$, measuring what fraction of distinguishability survives the channel: $\eta = 0$ destroys all information; $\eta = 1$ is lossless. Throughout, \emph{smaller $\eta$ means stronger contraction} (faster information loss).

\begin{lemma}[$\chi^2$ Strong Data Processing Inequality {\citep[see][]{raginsky2016strong, polyanskiy2017strong}}]
\label{lem:sdpi}
Let $K: \mathcal{Y} \to \Delta(\mathcal{Z})$ be a Markov kernel. For any distributions $P_Y$ and $Q_Y$ on $\mathcal{Y}$ with $P_Y \ll Q_Y$ and $\chi^2(P_Y \| Q_Y) < \infty$, define $P_Z = P_Y K$ and $Q_Z = Q_Y K$. Then
\[{
    \chi^2(P_Z \| Q_Z) \leq \eta_{\chi^2}(K) \cdot \chi^2(P_Y \| Q_Y).
}\]

\end{lemma}

For credit assignment: if $P_Y$ and $Q_Y$ represent state distributions under ``good step'' versus ``bad step,'' and $K$ represents subsequent reasoning, then outcome distributions are at most $\eta$ times as distinguishable as the original states. When we iterate through multiple steps, information loss compounds.

\begin{corollary}[Divergence Decay in Markov Chains {\citep[cf.][]{polyanskiy2017strong}}]
\label{cor:decay}
Consider a Markov chain $(S_0, S_1, \ldots, S_H)$ with transition kernels $K_0, K_1, \ldots, K_{H-1}$. For any two initial distributions $P_{S_0}$ and $Q_{S_0}$ with $P_{S_0} \ll Q_{S_0}$ and $\chi^2(P_{S_0} \| Q_{S_0}) < \infty$:
\[{
    \chi^2(P_{S_H} \| Q_{S_H}) \leq \left( \prod_{t=0}^{H-1} \eta_{\chi^2}(K_t) \right) \chi^2(P_{S_0} \| Q_{S_0}).
}\]

In particular, if $\eta_{\chi^2}(K_t) \leq \eta < 1$ for all $t$, then
\[{
    \chi^2(P_{S_H} \| Q_{S_H}) \leq \eta^H \cdot \chi^2(P_{S_0} \| Q_{S_0}).
}\]

\end{corollary}

\section{The Signal Decay Bound: Why Early Steps Are Hardest}
\label{sec:distinguish}

We now prove that information about step~$t$ decays exponentially as it propagates through the remaining $H-t$ stages, making credit assignment to early steps fundamentally harder than to late ones. We first introduce a state abstraction under which abstract dynamics are Markov (Section~\ref{subsec:abstraction}), show that the many-to-one structure of this abstraction induces contraction (Section~\ref{subsec:contraction}), and then prove the Signal Decay Bound, derive the critical horizon, and establish tightness (Section~\ref{subsec:signal-decay-bound}). \emph{Proofs appear in Appendix~\ref{app:signal-decay}.}

\subsection{State Abstraction and Abstract Dynamics}
\label{subsec:abstraction}

Credit assignment targets abstract content, not raw observations. In many sequential processes, the full observable state contains details (surface representation, sensor noise, formatting choices) that are irrelevant to the outcome. Multiple observable states may map to the same quality grade or semantic content. This many-to-one structure is the source of information loss, and formalizing it is the first step toward quantifying signal decay.

\begin{definition}[Markov State Abstraction]
\label{def:abstraction}
A \emph{Markov state abstraction} is a mapping $\phi: \mathcal{S} \to \mathcal{Z}$ from observable states to an abstract state space $\mathcal{Z}$ satisfying the \emph{lumpability condition}: for all $t$ and all $s, s' \in \mathcal{S}$ with $\phi(s) = \phi(s') = z$,
\[{
    \mathcal{L}(\phi(S_{t+1}) \mid S_t = s) = \mathcal{L}(\phi(S_{t+1}) \mid S_t = s').
}\]

The \emph{abstract state} at step $t$ is $Z_t = \phi(S_t)$.
\end{definition}

The lumpability condition, due to \citet{kemeny1960finite}, ensures that the abstract process $(Z_0, Z_1, \ldots, Z_H)$ is itself Markov: the distribution of the next abstract state depends only on the current abstract state, not on which observable state within the equivalence class was realized. When this condition holds, there exist well-defined abstract transition kernels $K_t: \mathcal{Z} \to \Delta(\mathcal{Z})$ with $K_t(z' \mid z) = \Prob(\phi(S_{t+1}) = z' \mid \phi(S_t) = z)$.

Lumpability is a modeling choice rather than a restrictive hypothesis: one \emph{defines} $\phi$ so that the abstract dynamics are Markov. If an initial choice of $\phi$ fails the condition, one refines the partition by augmenting $\phi$ with additional context until the condition holds, at the cost of a larger abstract space. The interesting question is not whether lumpability holds (it can always be arranged, at the cost of enlarging the abstract state space) but how much the resulting abstract kernels $K_t$ mix: do different abstract states lead to overlapping or distinct distributions over successor states? A coarser abstraction lumps more observable states together, increasing overlap and accelerating information loss. A finer abstraction preserves more distinctions but may require a larger state space. Section~\ref{subsec:contraction} formalizes this tradeoff. Section~\ref{sec:llm-application} in the Electronic Companion develops a concrete instantiation for language model reasoning.

\begin{example}[State Abstraction in Applications]
\label{ex:abstraction-applications}
\leavevmode
\begin{itemize}[leftmargin=*,itemsep=3pt,topsep=3pt]
    \item \textbf{Manufacturing:} The observable state $S_t$ consists of sensor readings at production stage $t$ (temperatures, pressures, vibration spectra). The abstract state $Z_t = \phi(S_t)$ captures the quality grade: ``within specification,'' ``marginal,'' or ``defective.'' Multiple sensor patterns map to the same grade.

    \item \textbf{AI reasoning:} The observable state $S_t$ is the token sequence generated after $t$ reasoning steps. The abstract state $Z_t = \phi(S_t)$ captures semantic content: the mathematical claims established, the proof strategy adopted, or the intermediate values derived. Multiple token sequences expressing the same content map to the same abstract state.
\end{itemize}
\end{example}

The lumpability condition requires exact agreement of abstract transition probabilities within each equivalence class. In practice, this may hold only approximately: two sensor patterns classified as the same quality grade may lead to slightly different distributions over downstream grades. We return to this point in Section~\ref{subsec:signal-decay-bound}, after establishing the main bound, where we show that approximate lumpability incurs only a bounded, typically negligible penalty.

\subsection{Why Abstract Dynamics Contract}
\label{subsec:contraction}

The many-to-one structure of $\phi$ introduced in Section~\ref{subsec:abstraction} causes information loss whenever distinct abstract states lead to overlapping next-state distributions. We now quantify this loss per step using the $\chi^2$-contraction coefficient $\eta_{\chi^2}(K) \in [0,1]$ (Definition~\ref{def:contraction}), which measures what fraction of $\chi^2$-divergence survives when distributions pass through a kernel $K$. When different abstract states $z$ and $z'$ produce overlapping distributions $K(\cdot \mid z)$ and $K(\cdot \mid z')$ over next states, $\eta_{\chi^2}(K) < 1$ and any statistical difference between two distributions over $Z_t$ shrinks after one transition.

To build intuition, suppose we want to determine whether the abstract chain was in state $z$ or state $z'$ at step $t$ by observing the state at step $t+1$. If $K(\cdot \mid z)$ and $K(\cdot \mid z')$ have disjoint support, observing $Z_{t+1}$ perfectly identifies $Z_t$, and no information is lost. But if both distributions assign positive probability to some state $y$, then observing $Z_{t+1} = y$ leaves residual uncertainty about whether $Z_t$ was $z$ or $z'$. The greater the overlap between $K(\cdot \mid z)$ and $K(\cdot \mid z')$, the more information is lost at each step.

\begin{example}[Contraction Coefficient in Extreme Cases]
\label{ex:contraction-extremes}
\leavevmode
\begin{itemize}[leftmargin=*,itemsep=3pt]
    \item \textbf{Deterministic dynamics ($\eta = 1$):} If each state $z$ transitions to a unique successor, the transition matrix is a permutation and $K(\cdot \mid z)$ and $K(\cdot \mid z')$ have disjoint support for $z \neq z'$. No information is lost: $\eta_{\chi^2}(K) = 1$.

    \item \textbf{Uniform mixing ($\eta = 0$):} If $K(\cdot \mid z) = \nu$ for all $z$ (every state transitions to the same distribution $\nu$), then $Z_{t+1}$ is independent of $Z_t$ and all information about the past is destroyed in one step: $\eta_{\chi^2}(K) = 0$.

    \item \textbf{Intermediate case:} Consider a two-state chain with $K(1 \mid 1) = 1-p$, $K(2 \mid 1) = p$, $K(1 \mid 2) = p$, $K(2 \mid 2) = 1-p$. The overlap between rows is $2\min\{p, 1-p\}$, and $\eta_{\chi^2}(K) = (1 - 2p)^2$. When $p = 0$, the chain is deterministic and $\eta = 1$; when $p = 1/2$, both rows are identical and $\eta = 0$.
\end{itemize}
\end{example}

The following two propositions formalize the relationship between row overlap and $\chi^2$-contraction, providing concrete sufficient conditions for $\eta < 1$. Both build on classical Markov chain mixing techniques (Doeblin decomposition, Dobrushin's ergodicity coefficient) combined with the SDPI framework (Lemma~\ref{lem:sdpi}); proofs appear in Appendix~\ref{app:signal-decay}.

\begin{proposition}[Contraction from Diversity]
\label{prop:diversity}
Consider a transition kernel $K$ on an abstract state space with $|\mathcal{Z}|$ states. If $K$ spreads probability such that $K(z' \mid z) \geq p_{\min} > 0$ for all state pairs $(z, z')$, then $\eta_{\chi^2}(K) \leq 1 - |\mathcal{Z}| \cdot p_{\min}$ (note $|\mathcal{Z}| \cdot p_{\min} \leq 1$ since rows sum to one).
\end{proposition}

Proposition~\ref{prop:diversity} provides a sufficient but strong condition: requiring $K(z' \mid z) \geq p_{\min}$ for all state pairs is often unrealistic. The following proposition gives a weaker condition based on the Dobrushin coefficient \citep{dobrushin1956central}, requiring only that each pair of rows has some overlap.

\begin{proposition}[Contraction from Row Overlap {\citep[cf.][]{dobrushin1956central, polyanskiy2017strong}}]
\label{prop:dobrushin}
For a transition kernel $K$ on a finite abstract state space $\mathcal{Z}$, define the Dobrushin coefficient:
\[
\alpha(K) := \min_{z, z' \in \mathcal{Z}} \sum_{y \in \mathcal{Z}} \min\{K(y \mid z), K(y \mid z')\}.
\]
Then $\eta_{\chi^2}(K) \leq 1 - \alpha(K)$.
\end{proposition}

\noindent The resulting bound on $\eta$ is looser than Proposition~\ref{prop:diversity} but applies more broadly, since it requires only pairwise row overlap rather than full minorization. In practice, $\alpha(K)$ can be estimated by sampling continuations from different prefixes and measuring their distributional overlap. In manufacturing, the same bound is computed from existing instrumentation or pilot audits: bin each stage's output into grades via $\phi$, estimate $K$ from empirical adjacent-stage transition frequencies, and compute $\alpha(K)$ from the resulting matrix, giving $\eta \leq 1 - \alpha(K)$ (Proposition~\ref{prop:dobrushin}). Example~\ref{ex:manufacturing-kernel} in Appendix~\ref{app:inspection-examples} illustrates this computation for a three-grade quality transition kernel.

Example~\ref{ex:contraction-extremes}, Example~\ref{ex:manufacturing-kernel} in Appendix~\ref{app:inspection-examples}, and Propositions~\ref{prop:diversity}--\ref{prop:dobrushin} show that contraction ($\eta < 1$) arises whenever the rows of the abstract kernel have nontrivial overlap. Contraction fails (i.e., $\eta = 1$ remains possible) only when some pairs of rows have disjoint support ($\alpha(K) = 0$), as in deterministic dynamics with distinct successors. We now formalize the assumption that contraction is strict.

\begin{assumption}[Contractive Dynamics]
\label{ass:contractive}
At each step $u$, the abstract transition kernel $K_u$ satisfies $\eta_{\chi^2}(K_u) \leq \eta$ for some $\eta < 1$.
\end{assumption}

Propositions~\ref{prop:diversity} and~\ref{prop:dobrushin} provide concrete upper bounds on $\eta$ in terms of the abstract kernel's row overlap. In Section~\ref{sec:llm-application} of the Electronic Companion, we establish how sampling temperature governs $\eta$ in language model reasoning, connecting the abstract contraction framework to a concrete and practically important design lever.
\subsection{The Signal Decay Bound}
\label{subsec:signal-decay-bound}

We now state our main result for this section: a lower bound on the number of i.i.d.\ outcome samples required to distinguish competing hypotheses about early steps. The bound formalizes the intuition that information decays exponentially as it propagates through the sequential chain. We formalize the hypotheses, prove the bound, derive the critical horizon, and show tightness.

Consider two hypotheses $H_0$ and $H_1$ about step $t$. These hypotheses induce different distributions $P^{(0)}_{Z_t}$ and $P^{(1)}_{Z_t}$ over the abstract state $Z_t$, and consequently different distributions over the binary outcome $R \in \{0,1\}$. We assume the hypotheses agree on everything else: the transition kernels $K_u$ for $u \geq t$ and the terminal map $g$ are identical under both. The question is: how many outcome observations are needed to reliably distinguish $H_0$ from $H_1$?

\begin{definition}[Hypothesis Model: Localized Distributional Difference]
\label{def:hypothesis-model}
Fix a step $t \in \{1, \ldots, H\}$. A pair of hypotheses $(H_0, H_1)$ \emph{differs only at step $t$} if the following generative model holds:
\begin{enumerate}[leftmargin=*,itemsep=2pt]
    \item \textbf{Common initial dynamics:} Under both hypotheses, the chain $(Z_0, \ldots, Z_{t-1})$ follows the same distribution determined by the initial state and transition kernels $K_0, \ldots, K_{t-2}$.
    \item \textbf{Localized divergence:} The conditional distribution of $Z_t$ given $Z_{t-1}$ differs:
    \[
    H_i: \quad Z_t \mid Z_{t-1} \sim K_{t-1}^{(i)}(\cdot \mid Z_{t-1}), \quad i \in \{0, 1\}.
    \]
    The induced marginal distributions satisfy $\chi^2(P^{(0)}_{Z_t} \| P^{(1)}_{Z_t}) = \Delta^2 > 0$.
    \item \textbf{Shared suffix (downstream dynamics):} For $u \geq t$, the transition kernels $K_u$ and terminal evaluation $g$ are identical under both hypotheses.
\end{enumerate}
\end{definition}

\noindent In applications: $H_0$ corresponds to ``step $t$ is valid'' (the transition $K_{t-1}^{(0)}$ produces correct states), while $H_1$ corresponds to ``step $t$ contains an error'' (the transition $K_{t-1}^{(1)}$ may produce erroneous states). In manufacturing: $H_0$ means stage $t$ operated correctly; $H_1$ means a defect was introduced at stage $t$. The shared-suffix condition isolates step $t$ as the sole source of difference: downstream stages process whatever they receive in the same way under both hypotheses, so any difference in the outcome distribution is attributable entirely to step $t$. Errors introduced at step $t$ may be corrected or may cascade through subsequent stages, but the stages themselves operate identically. Appendix~\ref{app:inspection-examples} discusses how $\Delta^2$ can be estimated from pilot audits in operations settings. We index step $t$ as the abstract state $Z_t$ immediately after the $t$-th transition, so $H-t$ counts the number of transitions remaining to reach the outcome.

\begin{definition}[Total Testing Error and Distinguishability]
\label{def:distinguish}
Let $H_0$ and $H_1$ induce distributions $P^{(0)}$ and $P^{(1)}$ over outcomes. The \emph{total testing error}, defined as the sum of type-I and type-II error probabilities minimized over all tests, is
\[{
    \mathrm{err}^*(P^{(0)}, P^{(1)}) := \inf_\psi \left( \Prob_{P^{(0)}}(\psi = 1) + \Prob_{P^{(1)}}(\psi = 0) \right) = 1 - d_{\mathrm{TV}}(P^{(0)}, P^{(1)}),
}\]
where the infimum is over all measurable tests $\psi: \mathcal{Y} \to \{0,1\}$. For $n$ i.i.d.\ samples, the $n$-sample total testing error is $\mathrm{err}^*_n := 1 - d_{\mathrm{TV}}(P^{(0)\otimes n}, P^{(1)\otimes n})$, where the test acts on $\mathcal{Y}^n$. The hypotheses are \emph{$(n, \epsilon)$-distinguishable} if $\mathrm{err}^*_n \leq \epsilon$.
\end{definition}

\noindent Thus $(n,\epsilon)$-distinguishability requires $d_{\mathrm{TV}}(P^{(0)\otimes n}, P^{(1)\otimes n}) \geq 1 - \epsilon$.

\begin{theorem}[Signal Decay Bound]
\label{thm:signal_decay}
Let $H_0$ and $H_1$ be two hypotheses differing only at step $t$, inducing distributions $P^{(0)}_{Z_t}$ and $P^{(1)}_{Z_t}$ over the abstract state with $P^{(0)}_{Z_t} \ll P^{(1)}_{Z_t}$ and $\chi^2(P^{(0)}_{Z_t} \| P^{(1)}_{Z_t}) = \Delta^2 < \infty$. Suppose the Markov state abstraction (Definition~\ref{def:abstraction}) and Assumption~\ref{ass:contractive} hold with contraction coefficient $\eta < 1$, and the binary outcome is determined by the terminal abstract state: $R = g(Z_H) \in \{0, 1\}$. Then:
\begin{enumerate}
    \item \textbf{Divergence decay:} The $\chi^2$-divergence at the outcome satisfies
    \[{
        \chi^2(P^{(0)}_R \| P^{(1)}_R) \leq \eta^{H-t} \cdot \Delta^2.
}\]

    \item \textbf{Sample complexity lower bound:} For any target error $\epsilon \in (0, 1/2)$, in the regime $\eta^{H-t} \Delta^2 \leq 1$, any test achieving total testing error (Definition~\ref{def:distinguish}) at most $\epsilon$ requires
    \[{
        n \geq \frac{(1-\epsilon)^2}{\eta^{H-t} \cdot \Delta^2}
}\]

    i.i.d.\ outcome samples. When $\eta^{H-t} \Delta^2 > 1$, this lower bound becomes $O(1)$; in this regime, the outcome distributions are well-separated and $O(1)$ samples suffice to achieve any constant total error rate.

    \item \textbf{Exponential growth:} For fixed initial divergence $\Delta^2 = \Theta(1)$, the sample complexity grows exponentially in distance to the outcome:
    \[{
        n = \Omega\left(\left(\frac{1}{\eta}\right)^{H-t}\right).
}\]

\end{enumerate}
\end{theorem}

\noindent Remark~\ref{rem:noisy-outcome} in Appendix~\ref{app:signal-decay} extends Theorem~\ref{thm:signal_decay} to noisy terminal observations $R \sim G(\cdot \mid Z_H)$.

Inverting Theorem~\ref{thm:signal_decay} yields a horizon formula: given a fixed sample budget $n$, how far back can one reliably evaluate?

\begin{corollary}[Critical Horizon]
\label{cor:horizon}
For initial divergence $\Delta^2$, sample budget $n$, and any target error $\epsilon \in (0, 1/2)$, a necessary condition for achieving total testing error at most $\epsilon$ is that the step lies within the $\epsilon$-critical horizon:
\[{
    H - t \leq H_{\mathrm{crit}}(\epsilon)
    := \max\!\left\{0,\;\frac{\ln\!\bigl(n \Delta^2/(1-\epsilon)^2\bigr)}{\ln(1/\eta)}\right\}.
}\]

Steps more than $H_{\mathrm{crit}}(\epsilon)$ from the outcome require exponentially many samples to evaluate from outcome labels alone. When $n\Delta^2 < (1-\epsilon)^2$, the horizon is zero: even the final step cannot be distinguished at error $\epsilon$. For intuition, $H_{\mathrm{crit}} \approx \ln(n\Delta^2)/\ln(1/\eta)$ when $\epsilon$ is a fixed constant.
\end{corollary}

\begin{example}[Critical Horizon in Practice]
\label{ex:horizon}
With $\eta = 0.9$, $\Delta^2 = 0.1$, and $n = 10^6$, the critical horizon is $H_{\mathrm{crit}} = \ln(10^5)/\ln(1/0.9) \approx 110$: one million outcome samples can reliably evaluate steps up to 110 steps before the outcome. With stronger contraction $\eta = 0.7$, the same budget yields $H_{\mathrm{crit}} \approx 32$. The formula depends logarithmically on $n$ and $\Delta^2$ but much more strongly on $\eta$: doubling $n$ extends the horizon by $\ln 2/\ln(1/\eta)$ steps ($\approx 6.6$ when $\eta=0.9$), while reducing $\eta$ from $0.9$ to $0.7$ cuts it by a factor of three.
\end{example}

The exponential rate in Theorem~\ref{thm:signal_decay} is achievable: the following proposition constructs a Bernoulli endpoint-testing instance in which the likelihood ratio test matches the lower bound up to constants.

\begin{proposition}[Achievability: Matching Upper Bound]
\label{prop:achievability}
Consider a stylized two-hypothesis problem where the outcome $R \in \{0,1\}$ is Bernoulli with parameters:
\begin{itemize}[leftmargin=*,itemsep=2pt]
\item Under $H_0$: $\Prob(R = 1) = p_0$
\item Under $H_1$: $\Prob(R = 1) = p_1 = p_0 + \delta$ where $\delta = \Theta(\eta^{(H-t)/2}\Delta)$
\end{itemize}
Assume $p_0 \in [\epsilon_0, 1-\epsilon_0]$ for some constant $\epsilon_0 > 0$, so the outcome is non-degenerate under both hypotheses. Then the $\chi^2$-divergence between outcome distributions is $\chi^2(P^{(0)}_R \| P^{(1)}_R) = \Theta(\eta^{H-t}\Delta^2)$, and the standard likelihood ratio test achieves total testing error $\epsilon$ with sample complexity $n = O(1/(\eta^{H-t}\Delta^2))$.
\end{proposition}

\noindent Together, Theorem~\ref{thm:signal_decay} and Proposition~\ref{prop:achievability} show that the exponential rate $(1/\eta)^{H-t}$ is both necessary and sufficient for the matched endpoint-testing problem: the scaling is tight up to constants. Example~\ref{ex:tight-chain} in the appendix constructs a full Markov-chain instance achieving this rate with equality, confirming that the barrier is a property of the information structure, not of our proof technique.

The preceding results assume exact lumpability (Definition~\ref{def:abstraction}), but in practice this may hold only approximately. Proposition~\ref{prop:approx-lumpability} in Appendix~\ref{app:signal-decay} shows that if each step introduces at most $\delta$ total-variation error relative to the lumpable idealization, the total-variation distance between outcome distributions acquires an additive penalty of $2(H-t)\delta$ on top of the signal-decay term $\sqrt{\eta^{H-t}\Delta^2/2}$. For a 50-stage process with per-step discrepancy $\delta = 0.001$, the accumulated penalty is at most $2 \times 50 \times 0.001 = 0.1$ in total variation distance. Approximate lumpability becomes a concern only for very deep processes or coarse abstractions, in which case the partition should be refined.

Theorem~\ref{thm:signal_decay} tells us how many samples are needed for reliable testing, and Corollary~\ref{cor:horizon} inverts this into a horizon formula. But neither quantifies how bad the error is when the sample budget falls short. The following corollary fills this gap: it gives an explicit lower bound on the minimax testing error, showing that when the available samples $n$ fall below the threshold $1/(\eta^{H-t}\Delta^2)$ required by Theorem~\ref{thm:signal_decay}, the error remains close to $1/2$; no test can do meaningfully better than random guessing.

\begin{corollary}[Statistical Limits of Value Estimation]
\label{cor:collapse}
Fix $\epsilon\in(0,1/2)$ and let $R^{1:n}$ denote $n$ i.i.d.\ outcome samples.
Under the Markov state abstraction (Definition~\ref{def:abstraction}) and Assumption~\ref{ass:contractive}, the minimax error for testing $H_0$ vs.\ $H_1$ based only on $R^{1:n}$ satisfies
\[
\inf_{\psi}\max_{i\in\{0,1\}}\Prob_i(\psi(R^{1:n})\neq i)
\;\ge\;
\frac12\left(1-\sqrt{\frac{(1+\eta^{H-t}\Delta^2)^n-1}{2}}\right).
\]
In particular, if
\[
n \;\le\; \frac{\ln\!\left(1+2(1-2\epsilon)^2\right)}{\eta^{H-t}\Delta^2},
\]
then every test $\psi$ has minimax error at least $\epsilon$.
\end{corollary}

\noindent Corollary~\ref{cor:collapse} uses the per-hypothesis error convention $\max_i \Prob_i(\psi \neq i)$, standard for Le~Cam bounds, whereas Definition~\ref{def:distinguish} uses the total error $\Prob_0(\psi=1)+\Prob_1(\psi=0)$; the two are related by a factor of two and yield equivalent scaling in $\eta^{H-t}\Delta^2$. When $n\eta^{H-t}\Delta^2$ is small, $(1 + \eta^{H-t}\Delta^2)^n \approx 1 + n\eta^{H-t}\Delta^2$, so the minimax error is close to $1/2$: the best possible test is barely better than a coin flip. For reinforcement learning, this means that critics trained from outcome labels cannot reliably evaluate steps beyond the critical horizon.

\section{Limits of Parallel Sampling}
\label{sec:width}

The Signal Decay Bound reveals that outcome signals weaken exponentially with depth. A natural response is to generate more samples: $W$ independent observations from the same starting condition, as in replicated testing or width-based methods like GRPO \citep{deepseek2024}. The intuition is that even if individual observations carry weak signal, averaging enough of them should recover the truth. This approach might seem to extract more signal from outcome data, despite Section~\ref{sec:distinguish} showing fundamental limits.

But can width truly overcome the exponential barrier? This section shows that it cannot: width reduces noise around an estimate but cannot amplify a signal that has already decayed. The statistical tools, variance reduction under averaging and design effects under equicorrelation \citep{kish1965survey}, are classical, but their interaction with the signal decay structure of Section~\ref{sec:distinguish} yields new limits on the reach of parallel sampling. We proceed in two steps: derive variance reduction bounds under independence (Section~\ref{subsec:multirollout}) and show that correlation limits effective width (Section~\ref{subsec:width_limit}). \emph{Proofs of all results in this section appear in Appendix~\ref{app:width}.}

\subsection{Multi-Rollout Value Estimation}
\label{subsec:multirollout}

Consider a learner who wishes to estimate the value $V^\pi(s) := \Prob_\pi(R = 1 \mid S_t = s)$, the probability of eventual success from state $s$ under the current policy $\pi$. (When conditioning on abstract state, replace $S_t$ with $Z_t$.) Rather than observing a single trajectory, the learner generates $W$ independent continuations from $s$ and observes their outcomes.

Throughout this section, we analyze the conditional distribution given a realized $s$; unconditional properties follow by averaging over the distribution of $s$.

\begin{definition}[Multi-Rollout Estimator]
\label{def:multirollout}
Given a state $s$ and $W$ rollout continuations $\tau_1, \ldots, \tau_W$ from $s$ under policy $\pi$, the \emph{multi-rollout value estimator} is:
\[{
    \hat{V}_W(s) := \frac{1}{W} \sum_{j=1}^W R_j,
}\]

where $R_j \in \{0, 1\}$ is the binary outcome of rollout $\tau_j$.
\end{definition}

The estimator is simply the empirical success rate across rollouts. Its statistical properties depend on the independence structure.

\begin{assumption}[Conditional Independence]
\label{ass:independence}
Given the conditioning variable $s$ and a fixed policy $\pi$, the rollouts $\tau_1, \ldots, \tau_W$ are conditionally independent:
\[{
    P(\tau_1, \ldots, \tau_W \mid s, \pi) = \prod_{j=1}^W P(\tau_j \mid s, \pi).
}\]

\end{assumption}

This assumption holds when each rollout is generated by independent sampling from the policy starting at $s$. It fails when rollouts share randomness, when the policy adapts based on previous rollouts, or when systematic biases cause rollouts to follow similar paths.

\begin{proposition}[Variance Reduction from Width]
\label{prop:width}
Under Assumption~\ref{ass:independence}, the multi-rollout estimator satisfies:
\begin{enumerate}
    \item \textbf{Unbiasedness:} $\mathbb{E}[\hat{V}_W(s) \mid s] = V^\pi(s)$.
    
    \item \textbf{Variance reduction:} $\mathrm{Var}(\hat{V}_W(s) \mid s) = \dfrac{V^\pi(s)(1 - V^\pi(s))}{W}$.
    
    \item \textbf{Concentration (Hoeffding):} With probability at least $1 - \delta$: $|\hat{V}_W(s) - V^\pi(s)| \leq \sqrt{\frac{\ln(2/\delta)}{2W}}$

\end{enumerate}
\end{proposition}

The variance in Part~2 is largest when $V^\pi(s) = 1/2$ (maximum uncertainty) and zero when outcomes are deterministic. Dividing by $W$ confirms the standard result: to halve the standard deviation, quadruple $W$. Part~3 follows from Hoeffding's inequality.

The preceding analysis assumes independent rollouts. In practice, rollouts from the same starting point are correlated, which further limits the benefit of width.

\subsection{Correlation Limits Effective Width}
\label{subsec:width_limit}

To quantify this limitation, we must account for correlation among rollouts from the same prefix. Consider collecting data from $n$ independent starting points, generating $W$ rollouts from each. Rollouts from \emph{different} starting points are independent, and their outcomes are uncorrelated. However, rollouts from the \emph{same} starting point share the same prefix and tend to make similar errors, inducing positive correlation among their outcomes. Correlation arises from shared prefix errors (all continuations inherit a latent early defect), mode collapse (low-entropy policies generate near-identical continuations, driving $\rho\to 1$), and systematic biases (learned error patterns activated by specific inputs).
When $W = 1$ (one rollout per starting point), we obtain $n$ i.i.d.\ outcome observations. When $W > 1$, this within-starting-point correlation reduces the effective sample size, as the following theorem quantifies.

\begin{theorem}[Effective Width Under Correlation {\citep[cf.][]{kish1965survey}}]
\label{thm:correlated}
Consider $W$ rollouts from a single starting point $s$. Assume the outcomes satisfy equicorrelation: $\mathrm{Corr}(R_i, R_j \mid s) = \rho$ for all $i \neq j$, where $\rho \in [0, 1)$ is a constant. (More generally, let $\rho$ be an upper bound on pairwise correlations.) The variance of the multi-rollout estimator $\hat{V}_W(s) = \frac{1}{W}\sum_{j=1}^W R_j$ satisfies:
\[{
    \mathrm{Var}(\hat{V}_W(s) \mid s) = \frac{\sigma^2(s)}{W} \cdot (1 + (W-1)\rho),
}\]

where $\sigma^2(s) = V^\pi(s)(1 - V^\pi(s))$. Equivalently, the correlated estimator has the same variance as an estimator using $W_{\mathrm{eff}}$ independent rollouts, where the \emph{effective width} is:
\[{
    W_{\mathrm{eff}} := \frac{W}{1 + (W-1)\rho}.
}\]

When $\rho = 0$ (perfect independence), $W_{\mathrm{eff}} = W$ and no degradation occurs. When $\rho > 0$, the effective width saturates as $W \to \infty$: $W_{\mathrm{eff}} \to 1/\rho$.
\end{theorem}

\begin{example}[Impact of Correlation]
\label{ex:correlation}
Suppose rollouts have correlation $\rho = 0.2$ (modest correlation from shared reasoning patterns). The effective width for various nominal widths:
\vspace{0.5\baselineskip}
\begin{center}
\begin{tabular}{c|ccccc}
$W$ & 10 & 50 & 100 & 500 & $\infty$ \\
\hline
$W_{\mathrm{eff}}$ & 3.6 & 4.6 & 4.8 & 5.0 & 5.0
\end{tabular}
\end{center}
\vspace{0.5\baselineskip}
Even with 500 rollouts, the effective width is only 5. Beyond roughly $1/\rho = 5$ rollouts, additional sampling provides negligible benefit.
\end{example}

The equicorrelation variance formula in Theorem~\ref{thm:correlated} is a classical design-effect result in survey sampling \citep{kish1965survey}; what is new is its implication for the credit assignment horizon, which we develop next. While the $W$ rollouts within each starting point are correlated, the $n$ starting points are independent. Define the \emph{effective sample size} as $N_{\mathrm{eff}} := n \cdot W_{\mathrm{eff}}$. Combined with the Signal Decay Bound (Theorem~\ref{thm:signal_decay}), distinguishing hypotheses about step $t$ requires
\[{
    N_{\mathrm{eff}} \;\gtrsim\; \frac{1}{\eta^{H-t}\Delta^2}.
}\]

The following corollary formalizes the resulting critical horizon.

\begin{corollary}[Critical Horizon with Width]
\label{cor:width-horizon}
Let $n$ be the number of independent starting points and $W_{\mathrm{eff}}$ the effective width per starting point. For fixed $\epsilon$, the critical horizon extends to
\[
H_{\mathrm{crit}} = \frac{\ln(n \cdot W_{\mathrm{eff}} \cdot \Delta^2)}{\ln(1/\eta)},
\]
where the $O(\ln(1/(1-\epsilon)))$ correction from Corollary~\ref{cor:horizon} is suppressed.
\end{corollary}

\noindent Width enters the critical horizon only logarithmically: doubling $W_{\mathrm{eff}}$ adds merely $\ln(2)/\ln(1/\eta) \approx 6.6$ steps when $\eta = 0.9$. And correlation caps the benefit: since $W_{\mathrm{eff}} \leq 1/\rho$, the effective sample size cannot exceed $n/\rho$ regardless of nominal width. The exponential dependence on depth cannot be escaped by increasing $W$.

In sum, width helps only when the step lies within the critical horizon and rollouts are approximately independent; beyond either limit, additional rollouts yield diminishing returns. For deep processes where $H > H_{\mathrm{crit}}$, width cannot substitute for step-level supervision; Appendix~\ref{app:budgeting} provides a decision rule: width alone is provably insufficient whenever $H > \ln(n\Delta^2/\rho)/\ln(1/\eta)$, and in practice one inspection dominates doubling $W$ whenever the maximal gap exceeds $H_{\mathrm{crit}}(\epsilon)$. This raises the question of how step-level supervision should be aggregated.

\section{From Outcome Limits to Process Supervision}
\label{sec:bellman}

The preceding sections established that outcome-based credit assignment faces an exponential barrier (Theorem~\ref{thm:signal_decay}) that width can only logarithmically mitigate (Section~\ref{sec:width}). For deep processes where $H > H_{\mathrm{crit}}$, this leaves only one path forward: \emph{step-level supervision}. Process reward models \citep{lightman2023lets}, intermediate checkpoints in manufacturing, and stage-gate reviews in project management all implement this principle. Step-level labels are not merely convenient; they are \emph{necessary}.

But step-level supervision introduces its own challenge: how should step-level feedback be aggregated into a training or evaluation objective? Standard reinforcement learning sums rewards across steps \citep{sutton2000policy}, but for sequential validity, where a single error invalidates the entire chain, this additive aggregation optimizes the wrong quantity. \emph{Proofs of all results in this section appear in Appendix~\ref{app:bellman}.}

\subsection{The Objective Mismatch}

Given that step-level supervision is necessary for deep processes, how should it be used? When training a system with step-level labels $r_t \in \{0,1\}$ indicating whether each step is valid, one can maximize either the expected number of correct steps or the probability that all steps are correct. These correspond, respectively, to additive and multiplicative aggregation of step-level feedback.

\begin{definition}[Step-Quality and Validity Objectives]
\label{def:two_objectives}
For a policy $\pi$ generating trajectories of length $H$ from initial state $S_0$, with step validity indicators $r_0, \ldots, r_{H-1}$:
\begin{enumerate}
    \item The \emph{step-quality objective} is $J_{\mathrm{add}} := \mathbb{E}_\pi\left[\sum_{t=0}^{H-1} r_t \,\big|\, S_0\right]$, the expected number of correct steps.
    \item The \emph{validity objective} is $J_{\mathrm{mult}} := \mathbb{E}_\pi\left[\prod_{t=0}^{H-1} r_t \,\big|\, S_0\right]$, the probability that all steps are correct.
\end{enumerate}
\end{definition}

\noindent Multiplicative objectives appear in risk-sensitive control \citep{howard1972risk} and stochastic reachability \citep{abate2008probabilistic}. Our contribution is identifying that \emph{sequential validity}, where a single failure invalidates the entire chain, naturally requires this structure, and analyzing the optimization challenges it creates. This applies to logical reasoning, manufacturing processes where any defective stage produces a defective product, and service operations where any failed touchpoint loses the customer.

Under independent steps with common success probability $p$, we have $J_{\mathrm{add}} = Hp$ while $J_{\mathrm{mult}} = p^H$. Even at 99\% per-step accuracy over 100 steps, $J_{\mathrm{add}} = 99$ while $J_{\mathrm{mult}} = 0.37$: systems optimizing additive objectives learn to produce sequences that are mostly correct but nonetheless invalid.

\subsection{Gradient Decay in the Multiplicative Objective}
\label{subsec:optimization}

The preceding analysis suggests using the multiplicative objective directly. However, this approach encounters a fundamental optimization barrier: the multiplicative objective exhibits vanishing gradients for long chains with imperfect per-step accuracy.

Under the independence model, differentiating $J_{\mathrm{add}} = Hp$ and $J_{\mathrm{mult}} = p^H$ with respect to policy parameters $\theta$ gives $\nabla_\theta J_{\mathrm{add}} = H \nabla_\theta p$ and $\nabla_\theta J_{\mathrm{mult}} = H p^{H-1} \nabla_\theta p$. The multiplicative gradient contains an additional factor $p^{H-1}$ that attenuates the learning signal. For $p = 0.95$ over a 100-step chain, this factor equals $0.95^{99} \approx 0.006$, rendering the validity objective nearly unlearnable. The barrier becomes severe precisely when it matters most: for longer chains and policies that have not yet achieved high per-step accuracy.

This gradient decay is conceptually distinct from the signal decay of Section~\ref{sec:distinguish}. Signal decay concerns information flow: outcome data becomes uninformative about early steps. Gradient decay concerns objective structure: the multiplicative objective attenuates learning signal regardless of information availability. A complete training framework must address both barriers. We propose a curriculum approach: begin with the learnable additive objective, then transition to the correct multiplicative objective as policy quality improves \citep{bengio2009curriculum}.

\begin{definition}[Interpolated Objective]
\label{def:interpolated}
For $\lambda \in [0,1]$, the \emph{interpolated objective} is:
\[{
    J_{\lambda}(\theta) := (1-\lambda) \cdot J_{\mathrm{add}}(\theta) + \lambda \cdot J_{\mathrm{mult}}(\theta).
}\]

\end{definition}

\noindent At $\lambda = 0$, the objective is learnable but incorrect; at $\lambda = 1$, correct but potentially unlearnable. The following result characterizes how $\lambda$ should evolve with policy quality.

\begin{proposition}[Gradient-Preserving Curriculum]
\label{prop:curriculum}
Assume $\nabla_\theta J_{\mathrm{add}}$ and $\nabla_\theta J_{\mathrm{mult}}$ have non-negative inner product. To maintain $\|\nabla_\theta J_{\lambda}\| \geq c \cdot \|\nabla_\theta J_{\mathrm{add}}\|$ for some $c \in (0,1)$, it suffices to set $\lambda \leq 1 - c$.
\end{proposition}

\noindent The proposition establishes that the interpolated objective preserves gradient signal when $\lambda$ is sufficiently small. In the independence model where $J_{\mathrm{add}} = Hp$ and $J_{\mathrm{mult}} = p^H$, the gradients satisfy $\|\nabla_\theta J_{\mathrm{mult}}\| = p^{H-1} \|\nabla_\theta J_{\mathrm{add}}\|$. As the policy improves and per-step success probability $p$ approaches $1$, this attenuation factor approaches $1$, allowing $\lambda$ to increase toward $1$ without sacrificing gradient magnitude. This yields a curriculum: begin training with $\lambda$ near $0$ to ensure learnability, then gradually increase $\lambda$ toward $1$ as the policy improves, transitioning from a learnable but misspecified objective to one that is correctly specified.

Having established \emph{why} step-level supervision is necessary and \emph{how} to aggregate it, the next section addresses \emph{where}: given an inspection budget, which steps should be observed?

\section{Optimal Intermediate Inspection Design}
\label{sec:inspection-design}

We now convert the information-theoretic limits into \emph{optimal intermediate inspection schedules}. The organizing principle is that inspections shorten the information path that credit assignment must traverse: instead of propagating from step $t$ to the terminal outcome at time $H$, it suffices to propagate from $t$ to the next checkpoint.

The inspection design problem has a long history in OR under different guises. In manufacturing, it is the placement of quality checkpoints along a production line \citep{lindsay1964optimal, raz1986economic}. In sequential testing, it is related to adaptive group testing \citep{du2000combinatorial}. Our contribution is to solve inspection placement under the specific signal-decay structure implied by the Strong Data Processing Inequality, yielding closed-form schedules for both homogeneous and heterogeneous contraction.

We proceed in three steps. First, we formalize the inspection feedback model and extend the signal-decay bound to an arbitrary inspection time $u$. Second, under homogeneous contraction we show that worst-case sample complexity is governed by a single scalar, the largest distance from any step to its next inspection, and we solve the resulting placement problem (uniform spacing). Third, under heterogeneous contraction we show that feasibility depends on cumulative information loss along each segment, and we solve the minimum-inspection design problem via a greedy information-distance schedule. Cost and sample tradeoffs are then handled in Section~\ref{subsec:budgeted-design}. Proofs of all results in this section appear in Appendix~\ref{app:inspection}.

\subsection{Inspection Schedules and Feedback Model}
\label{subsec:inspection-model}

Recall the abstract Markov chain $(Z_0,\dots,Z_H)$ induced by the abstraction $\phi$ and policy $\pi$.
We treat the terminal outcome $R=g(Z_H)\in\{0,1\}$ as a \emph{terminal inspection} at time $H$.

\begin{definition}[Inspection Schedule]
\label{def:inspection-schedule}
An \emph{inspection schedule} is a strictly increasing set of inspection times
\[
\mathcal{T}=\{t_1<t_2<\cdots<t_m\}\subseteq\{1,2,\dots,H-1\}.
\]
We use the augmented sequence $0=:t_0<t_1<\cdots<t_m<t_{m+1}:=H$.
The number of paid intermediate inspections is $m$.
\end{definition}

\begin{definition}[Inspection Observations]
\label{def:inspection-observation}
At each inspection time $u\in\mathcal{T}\cup\{H\}$, the learner observes an \emph{inspection output}
$Y_u := g_u(Z_u)\in\mathcal{Y}_u$,
where $g_u$ is an external evaluator (human, verifier, metrology device) applied to the abstract
state. The terminal outcome is the special case $Y_H\equiv R=g(Z_H)$.
\end{definition}

The model allows many inspection types: prefix-validity checks, milestone verification, partial correctness scores, or any measurement that is a function of the abstract state. Remark~\ref{rem:imperfect-inspection} in Appendix~\ref{app:inspection} extends the model to noisy inspections with false positives and negatives, adding an observation-channel penalty to the segment attenuation budget. The results below are information-theoretic and do not assume any particular learning algorithm.

\subsection{Signal Decay to Intermediate Inspections}
\label{subsec:inspection-signal-decay}

The design lever is that intermediate checkpoints reduce the effective distance over which information must propagate. If step $t$ is followed by an inspection at time $u<H$, then credit assignment for step $t$ need only overcome $u-t$ steps of attenuation rather than $H-t$. We formalize this for inhomogeneous contraction (step-dependent mixing rates); the homogeneous case follows in Section~\ref{subsec:homogeneous-opt}.

\begin{assumption}[Inhomogeneous Contractive Dynamics]
\label{ass:inhom-contraction}
Let $K_t$ denote the abstract transition kernel from $Z_t$ to $Z_{t+1}$, with $\chi^2$-contraction coefficient $\eta_{\chi^2}(K_t) \leq \eta_t \in (0,1]$ for $t = 0, 1, \ldots, H-1$. Define the cumulative attenuation from step $t$ to inspection time $u>t$ as
\[
\Attn(t \to u) := \prod_{j=t}^{u-1} \eta_j.
\]
In the homogeneous case $\eta_t \equiv \eta$, this simplifies to $\Attn(t \to u) = \eta^{u-t}$.
\end{assumption}

\begin{proposition}[Signal Decay to an Intermediate Inspection]
\label{prop:inspection-decay}
Consider two hypotheses $H_0$ and $H_1$ differing only at step $t$, inducing abstract state distributions with $\chi^2(P^{(0)}_{Z_t} \| P^{(1)}_{Z_t}) = \Delta^2$. Under Assumption~\ref{ass:inhom-contraction}, the $\chi^2$-divergence at any downstream inspection $Y_u = g_u(Z_u)$ satisfies
\[
\chi^2(P^{(0)}_{Y_u} \| P^{(1)}_{Y_u}) \leq \Attn(t \to u)\cdot \Delta^2.
\]
\end{proposition}

Proposition~\ref{prop:inspection-decay} implies that an inspection at time $u$ replaces the terminal attenuation factor $\Attn(t\to H)$ by the shorter-horizon factor $\Attn(t\to u)$. The following corollary makes the value of intermediate inspection precise: placing a checkpoint at time $u$ reduces the attenuation exponent from $H-t$ to $u-t$. The remaining question is how to place inspection times so that no step is too far from its next checkpoint.

\begin{corollary}[Sample Complexity with Intermediate Inspection]
\label{cor:inspection-sample-lb}
Under the conditions of Proposition~\ref{prop:inspection-decay}, any test distinguishing $H_0$ from $H_1$ at error rate $\epsilon$ requires
\[
 n \geq \frac{(1-\epsilon)^2}{\Attn(t \to u)\cdot \Delta^2}.
\]
In the homogeneous case, $n \geq (1-\epsilon)^2 / (\eta^{u-t} \Delta^2)$.
\end{corollary}

\subsection{Homogeneous Contraction: Minimax-Optimal Placement}
\label{subsec:homogeneous-opt}

Under homogeneous contraction, the sample complexity penalty depends only on the distance from a step to its next inspection (Corollary~\ref{cor:inspection-sample-lb}). This turns inspection design into a one-dimensional placement problem on $\{0,1,\ldots,H\}$ with a single governing design statistic.

The development in this subsection follows a short chain of implications. We first define the distance of each step to its next checkpoint and the induced maximal gap $L(\mathcal{T})$. We then show that any schedule must incur a worst-case sample complexity that depends exponentially on $L(\mathcal{T})$. Finally, we solve the placement problem by minimizing $L(\mathcal{T})$ over schedules with a fixed number of inspections $m$ and translate the result into a minimum number of inspections required for feasibility.

\begin{definition}[Downstream Distance and Maximal Gap]
\label{def:downstream-gap}
For a schedule $\mathcal{T} = \{t_1 < \cdots < t_m\}$ with augmented times $t_0 = 0$ and $t_{m+1} = H$, the \emph{downstream distance} of step $t$ is the distance to its next inspection:
\[
 d_{\mathcal{T}}(t) := \min\{u - t : u \in \mathcal{T} \cup \{H\},\ u > t\}.
\]
The \emph{maximal gap} is $L(\mathcal{T}) := \max_{t \in \{0, \ldots, H-1\}} d_{\mathcal{T}}(t)$, equivalently $L(\mathcal{T}) = \max_{i=0,\ldots,m}(t_{i+1} - t_i)$.
\end{definition}

The maximal gap determines worst-case performance: the step at the beginning of the longest segment between inspections faces the most attenuation before any observation. For a fixed number of intermediate inspections $m$, the minimax placement problem is therefore
\[
\min_{|\mathcal{T}|=m} L(\mathcal{T}),
\]
because $L(\mathcal{T})$ controls the worst-case attenuation exponent (and hence the worst-case sample complexity) across all steps. Operationally, $L(\mathcal{T})$ is the length of the longest uninspected ``run'' of steps; the minimax goal is to prevent any stage from being too far from a checkpoint.

\begin{proposition}[Worst-Case Sample Complexity]
\label{prop:largest-gap-lb}
Under homogeneous contraction $\eta_t \equiv \eta < 1$, any inspection schedule $\mathcal{T}$ admits a step $t^\star$ such that distinguishing hypotheses about $t^\star$ requires
\[
 n \geq \frac{(1-\epsilon)^2}{\eta^{L(\mathcal{T})} \Delta^2}
\]
samples.
\end{proposition}

Proposition~\ref{prop:largest-gap-lb} identifies the correct minimax objective: reduce the largest gap to reduce worst-case sample complexity. The remaining question is purely combinatorial: how should $m$ inspections partition a horizon of length $H$ to minimize the length of the longest segment?

\begin{theorem}[Minimax-Optimal Inspection Placement]
\label{thm:uniform-optimal}
Under homogeneous contraction $\eta_t \equiv \eta < 1$, the minimal achievable maximal gap with $m$ intermediate inspections is
\[
\min_{|\mathcal{T}|=m} L(\mathcal{T}) = \Big\lceil \frac{H}{m+1} \Big\rceil.
\]
This minimum is achieved by near-uniform spacing: placing inspections at $t_i = \lfloor iH/(m+1) \rfloor$ for $i = 1, \ldots, m$.
\end{theorem}

While uniform spacing is a classical heuristic for worst-case gap objectives, our contribution is identifying the \emph{correct design objective}, namely the gap in SDPI attenuation units, and linking it to sample complexity via the Signal Decay Bound. The implication is that the detailed location of checkpoints is secondary; what matters is controlling the longest uninspected run. Combining Proposition~\ref{prop:largest-gap-lb} and Theorem~\ref{thm:uniform-optimal} yields an explicit feasibility criterion.

\begin{corollary}[Sample Complexity Under $m$ Inspections]
\label{cor:minimax-lb-m}
With $m$ minimax-optimal inspections, the worst-case sample complexity for credit assignment is
\[
 n \geq \frac{(1-\epsilon)^2}{\eta^{\lceil H/(m+1) \rceil} \Delta^2}.
\]
\end{corollary}

\begin{corollary}[Minimum Inspections for Feasibility]
\label{cor:min-inspections}
For all steps to be $(n, \epsilon)$-distinguishable, the number of inspections must satisfy
\[
 m \geq \Big\lceil \frac{H}{H_{\mathrm{crit}}(\epsilon)} \Big\rceil - 1,
\]
where $H_{\mathrm{crit}}(\epsilon) = \bigl(\ln(n\Delta^2) - 2\ln(1-\epsilon)\bigr) / \ln(1/\eta)$ is the $\epsilon$-critical horizon. This is a necessary condition; the sufficient number may be one higher due to integer rounding when $H/(m+1)$ is not an integer.
\end{corollary}

\noindent Example~\ref{ex:semiconductor} in Appendix~\ref{app:inspection-examples} applies Corollary~\ref{cor:min-inspections} to a 50-stage semiconductor fabrication process, showing that a single midpoint inspection suffices to bring all stages within the critical horizon.

\subsection{Heterogeneous Contraction: Minimum Inspections via Information Distance}
\label{subsec:heterogeneous-opt}

Uniform spacing is optimal when all steps contract equally, but many processes exhibit depth-dependent contraction. In such settings, the design objective shifts from ``minimize the largest segment length'' to a feasibility-driven problem: choose as few inspections as possible so that every segment has cumulative information loss at most the threshold implied by the sample budget.

\begin{definition}[Information Distance]
\label{def:info-distance}
Under Assumption~\ref{ass:inhom-contraction}, the \emph{per-step information distance} is $w_t := \ln(1/\eta_t) \geq 0$. For an interval $[a, b)$, the \emph{cumulative information distance} is
\[
 W([a,b)) := \sum_{t=a}^{b-1} w_t = \ln(1/\Attn(a \to b)).
\]
Steps with strong contraction (small $\eta_t$) contribute large $w_t$; steps with weak contraction contribute little.
\end{definition}

Because $\Attn(t \to u)=\exp(-W([t,u)))$, Corollary~\ref{cor:inspection-sample-lb} yields a segment-level feasibility threshold.

\begin{lemma}[Per-Segment Feasibility]
\label{lem:segment-feasible}
Fix consecutive inspection times $t_i < t_{i+1}$. If
\[
 W([t_i, t_{i+1})) \leq \Gamma(n, \Delta^2, \epsilon) := \ln(n\Delta^2) - 2\ln(1-\epsilon),
\]
then every step $t \in \{t_i,\ldots,t_{i+1}-1\}$ is $(n,\epsilon)$-distinguishable using the inspection at $t_{i+1}$. Conversely, if step $t_i$ is $(n,\epsilon)$-distinguishable using the inspection at $t_{i+1}$, then the displayed inequality must hold.
\end{lemma}

The heterogeneous inspection design problem is to partition the horizon into segments satisfying the Lemma~\ref{lem:segment-feasible} constraint while using as few paid inspections as possible. A greedy construction is optimal.

\begin{theorem}[Greedy Inspection is Optimal]
\label{thm:greedy-opt}
Among all schedules satisfying $W([t_i, t_{i+1})) \leq \Gamma$ for each segment, the following greedy construction minimizes the number of inspections: starting at $t_0 = 0$, place each inspection $t_{i+1}$ at the largest index such that $W([t_i, t_{i+1})) \leq \Gamma$.
\end{theorem}

Greedy segmentation is standard for interval-packing problems; what the SDPI analysis adds is the metric of cumulative information distance $W([t_i,t_{i+1})) = \sum_t \ln(1/\eta_t)$ rather than step count, under which feasibility is defined, so inspections concentrate automatically where contraction is strongest. When $\eta_t \equiv \eta$, we have $w_t \equiv \ln(1/\eta)$, so greedy produces segments of length $\lfloor \Gamma / \ln(1/\eta) \rfloor$, recovering uniform spacing (Theorem~\ref{thm:uniform-optimal}). When contraction varies, greedy produces non-uniform placement: segments spanning high-contraction steps are short; segments spanning low-contraction steps are long. Example~\ref{ex:service} in Appendix~\ref{app:inspection-examples} applies Theorem~\ref{thm:greedy-opt} to a 50-touchpoint customer journey with depth-dependent contraction, showing that the greedy algorithm places its single inspection in the high-contraction early stages.

\subsection{Joint Sample-Inspection Budgeting}
\label{subsec:budgeted-design}

Inspection frequency trades off per-trajectory cost against the exponential sample savings from shortening information paths.
Let $c_{\mathrm{out}}$ denote the per-trajectory cost of terminal evaluation and $c_{\mathrm{insp}}$ the per-trajectory cost per intermediate inspection.
Collecting $n$ trajectories with $m$ inspections each has total budget $B=n(c_{\mathrm{out}}+m c_{\mathrm{insp}})$.

\begin{proposition}[Budget Lower Bound]
\label{prop:budget-lb}
Under homogeneous contraction, the minimum budget required for worst-case $(n, \epsilon)$-distinguishability with $m$ inspections is
\[
 B \geq (c_{\mathrm{out}} + m \cdot c_{\mathrm{insp}}) \cdot \frac{(1-\epsilon)^2}{\eta^{\lceil H/(m+1)\rceil}\Delta^2}.
\]
\end{proposition}

Increasing $m$ reduces the sample requirement exponentially but raises per-trajectory cost linearly. A practical rule is to choose the smallest $m$ that makes the uniform segment length $\lceil H/(m+1)\rceil$ fall within $H_{\mathrm{crit}}(\epsilon)$ (Corollary~\ref{cor:min-inspections}).

\begin{corollary}[Inspection Density for Polynomial Complexity]
\label{cor:poly-density}
To achieve polynomial sample complexity $n = O(H^p)$ under homogeneous contraction, the number of inspections must scale as $m = \Omega(H / \log H)$.
\end{corollary}

\noindent Corollary~\ref{cor:poly-density} implies that inspection effort must grow nearly linearly with process depth to avoid exponential sample requirements. An implementable inspection design procedure and detailed cost--sample tradeoff analysis are provided in the Electronic Companion (Appendix~\ref{app:budgeting}).

\section{Numerical Validation}
\label{sec:numerical}

We validate the paper's five core predictions through synthetic Markov chains with controlled parameters and, where applicable, LLM reasoning chains on GSM8K (a grade-school math benchmark). Across all experiments, observed behavior closely matches the theoretical bounds, confirming that the closed-form results provide reliable quantitative guidance in practice. Table~\ref{tab:numerics_headline} summarizes the headline findings; full experimental setups, parameter configurations, and all figures are provided in the Electronic Companion (Appendix~\ref{app:numerical-details}).

\begin{table}[t]
\TABLE
{Headline numerical findings. Full setups and plots in Appendix~\ref{app:numerical-details}.\label{tab:numerics_headline}}
{\small
\begin{tabular}{p{0.27\textwidth} p{0.29\textwidth} p{0.35\textwidth}}
\toprule
\textbf{Prediction} & \textbf{Test bed} & \textbf{Key finding} \\
\midrule
Signal decay (Thm.~\ref{thm:signal_decay}) &
10-state synthetic chains, $\eta\in\{0.7,0.8,0.9,0.95\}$ &
Measured $\chi^2$-divergences match $\eta^{H-t}\Delta^2$ exactly ($R^2>0.999$). \\
\addlinespace
Effective width (Thm.~\ref{thm:correlated}) &
Synthetic ($\rho=0.15$); LLM rollouts on GSM8K ($\rho\approx 0.49$) &
$W_{\mathrm{eff}}$ saturates at $6.6$ vs.\ theoretical $1/\rho=6.7$ (synthetic); at $2.2$ vs.\ $2.0$ (LLM). \\
\addlinespace
Uniform inspection (Thm.~\ref{thm:uniform-optimal}) &
$H=20$, $\eta=0.9$, $m=3$ checkpoints &
Uniform spacing achieves 40--47\% lower worst-case error than front-loaded, back-loaded, or random schedules. \\
\addlinespace
Critical horizon (Cor.~\ref{cor:horizon}) &
Synthetic, $\eta\in\{0.7,0.8\}$, $n=1{,}000$ &
Attribution accuracy drops to chance ($50\%$) at or before the predicted $H_{\mathrm{crit}}$. \\
\addlinespace
Objective mismatch (Def.~\ref{def:two_objectives}) &
50 LLM reasoning chains on GSM8K &
16\% of chains are ``mostly correct'' (additive reward $\geq0.80$) yet produce wrong answers. \\
\bottomrule
\end{tabular}}
{}
\end{table}

\section{Conclusion}
\label{sec:conclusion}

This paper establishes information-theoretic limits on credit assignment in deep sequential processes, revealing a sharp sample-complexity threshold that constrains both inspection system design and reasoning model training.

\paragraph{The core insight.}
The Signal Decay Bound (Theorem~\ref{thm:signal_decay}) shows that information about step $t$ decays as $\eta^{H-t}$, defining a critical horizon $H_{\mathrm{crit}} \approx \ln(n\Delta^2)/\ln(1/\eta)$ (up to the $\epsilon$-correction in Corollary~\ref{cor:horizon}).
Steps within $H_{\mathrm{crit}}$ of the outcome can be evaluated from endpoint data; steps beyond require exponentially many samples in the distance $H-t$.

\paragraph{Why this matters for AI reasoning.}
The framework explains why outcome-trained value networks collapse to near-constant predictions on long reasoning chains \citep{vineppo2024}: this is statistically optimal behavior when outcome labels carry negligible information about early steps. Token-level dynamics are deterministic, but credit assignment targets semantic content, and the many-to-one abstraction from tokens to meaning is where contraction arises (see Section~\ref{sec:llm-application} in the Electronic Companion for a detailed treatment). Process reward models work by providing signal beyond the critical horizon; width-based methods like GRPO \citep{deepseek2024} face diminishing returns because correlation caps effective width at $1/\rho$ (Theorem~\ref{thm:correlated}).

\paragraph{Beyond AI: unifying inspection design.}
The same mathematics applies across operations settings (manufacturing, service operations, supply chains) wherever sequential processes are evaluated at endpoints. The critical horizon formula provides a direct diagnostic for whether intermediate inspection is necessary, and the optimal scheduling results (Section~\ref{sec:inspection-design}) show that uniform spacing is minimax-optimal under homogeneous contraction, with a greedy algorithm handling heterogeneous contraction. Achieving polynomial sample complexity requires inspection density scaling as $\Omega(H/\log H)$ (Corollary~\ref{cor:poly-density}). A separate failure mode arises from objective structure: standard additive reward aggregation optimizes expected correct steps rather than end-to-end validity (Section~\ref{sec:bellman}), and the curriculum of Proposition~\ref{prop:curriculum} bridges this gap by interpolating between learnable and correct objectives as policy quality improves.


\paragraph{Limitations and future directions.}
Key extensions include data-driven estimation of $\eta$ and $\Delta^2$, learning the abstraction $\phi$ from observations, and moving from binary to graded validity.

\noindent By establishing what endpoint data can and cannot reveal, this work provides a common analytical foundation for inspection design in operations and supervision design in AI.

\section{Code and Data Disclosure}\label{sec:code-data}
The code and data to support the numerical experiments in this paper will be made publicly available
upon acceptance.

\bibliographystyle{informs2014}
\bibliography{references}

\clearpage
\setcounter{page}{1}
\renewcommand{\thepage}{EC\arabic{page}}
\raggedbottom

\begin{APPENDICES}
\section*{Electronic Companion}

\section{Extended Literature Review}
\label{app:litreview}
The information-theoretic structure underlying credit assignment connects five research streams that have developed largely independently: inspection design in operations research, credit assignment in AI reasoning, reinforcement learning theory, information-theoretic analysis of sequential processes, and the emerging application of OR tools to AI system design.

\paragraph{Inspection design in operations research.}
The optimal placement of inspection points in sequential production is a foundational topic in operations research. \citet{lindsay1964optimal} formulated inspection allocation as a serial dynamic program, minimizing total cost of inspection plus penalties for escaping defects. \citet{raz1986economic} surveyed extensions incorporating multiple inspection stations and variable costs; \citet{tapiero1987quality} introduced imperfect inspection and learning effects. Recent work incorporates adaptive inspection \citep{zou2023adaptive}, predictive maintenance \citep{compare2020predictive}, and machine-learning-based prognosis \citep{del2019review}. \citet{ghuge2024nonadaptive} study sequential testing of system components under budget constraints, providing a constant-factor approximation for nonadaptive score classification. \citet{duenyas2024group} analyze supply chains where components cannot be individually tested, developing coordination mechanisms when stage-level inspection is infeasible, precisely the terminal-only observation regime our theory addresses.

This literature optimizes checkpoint placement to minimize cost under the assumption that defect signals persist through the production chain. Our work identifies a structural limit: in deep processes, signal decay makes observability, not cost, the binding constraint. The critical horizon formula quantifies when this regime applies.

\paragraph{Credit assignment in AI reasoning.}
The credit assignment problem was recognized by \citet{minsky1961steps} as fundamental to learning systems, yet has resisted systematic theoretical treatment. Our framework provides an information-theoretic characterization of when outcome supervision suffices versus when process supervision becomes necessary, grounded in the signal decay structure of deep sequential processes.

Recent work explores process supervision as a remedy: \citet{lightman2023lets} show that step-level labels substantially improve mathematical reasoning; \citet{uesato2022solving} compare outcome and process supervision for mathematical reasoning; \citet{wang2024math} develop automated step-labeling methods. These studies demonstrate that process supervision helps but do not characterize \emph{when} it becomes necessary. Our Signal Decay Bound fills this gap.

Empirically, value networks trained on outcome data exhibit ``critic collapse'', converging to near-constant predictions on long chains \citep{vineppo2024}. Width-based methods such as GRPO \citep{deepseek2024} address this by replacing learned critics with Monte Carlo estimation. Our framework explains both phenomena: critic collapse reflects the absence of signal beyond the critical horizon, and width-based methods face diminishing returns because correlation caps effective width.

\paragraph{Reinforcement learning theory.}
Standard reinforcement learning theory establishes polynomial sample complexity under regularity conditions such as bounded concentrability \citep{kakade2003sample, azar2013minimax}. Recent work by \citet{jia2025verify} shows that outcome and process supervision are polynomially equivalent when concentrability is bounded. Our setting falls outside this regime: in deep sequential processes where each step may fail, concentrability grows exponentially in horizon, and the polynomial equivalence breaks down. This is precisely why the critical horizon matters. The exponential barrier in this regime has also been identified by independent and concurrent work: \citet{chen2025exponential} establish an exponential separation between outcome-based and per-step feedback for certain MDPs, using function approximation complexity arguments. Their construction demonstrates the separation through a worst-case MDP family, while our SDPI-based analysis provides a continuous diagnostic ($H_{\mathrm{crit}}$) and prescriptive inspection design applicable to any given process; the two approaches are complementary. Sample complexity in related RL settings has received attention in the OR community: \citet{hu2024fast} establish fast convergence rates for offline RL, \citet{cheung2023nonstationary} develop algorithms for nonstationary environments, and \citet{lykouris2025corruption} address exploration under adversarial corruption. These works assume bounded concentrability or benign nonstationarity; the deep sequential regime we study, where signal decays exponentially, falls outside their scope.

A second point of departure from standard reinforcement learning concerns objective structure. Standard algorithms aggregate rewards additively, but sequential validity requires all steps to be correct, inducing a multiplicative objective. This connects to risk-sensitive MDPs \citep{howard1972risk}, multiplicative-reward models \citep{baier2025multiplicative}, and minimax-optimal risk-sensitive RL \citep{wang2023near}. We propose a curriculum strategy that interpolates between the learnable additive objective and the correct multiplicative one.

\paragraph{Information-theoretic foundations.}
Filter stability results \citep{del2001stability} establish that dependence on initial conditions decays asymptotically in Markov chains, but do not yield the non-asymptotic, closed-form rates that credit assignment requires: a practitioner needs to know how many samples suffice for a process of a given depth, not merely that the problem becomes harder in the limit. Our analysis provides such rates by combining two classical tools. Strong Data Processing Inequalities \citep{polyanskiy2017strong, raginsky2016strong} quantify per-step information loss, establishing that $\chi^2$-divergence contracts by a factor $\eta < 1$ at each transition. Le Cam's method \citep{lecam1973convergence} then translates the resulting divergence bounds into sample complexity lower bounds, yielding the critical horizon formula. Recent extensions to bandits under communication constraints \citep{pensia2024bandits} and interactive learning \citep{chen2024unified} address related but distinct settings.

\paragraph{Operations research tools for AI system design.}
A growing body of work applies OR methodology directly to AI systems. \citet{snell2024scaling} frame inference-time compute allocation as an optimization problem, adaptively distributing verification effort across prompts of varying difficulty, an allocation problem that parallels inspection placement at the problem-structure level, though without the information-theoretic signal decay analysis we develop here. \citet{besta2025survey} formulate LLM reasoning explicitly within the MDP framework, treating chain-of-thought generation as sequential decision-making. \citet{zhang2024rest} apply Monte Carlo Tree Search to guide multi-step reasoning with step-level evaluation, and \citet{tang2024code} model iterative code repair as a multi-armed bandit. \citet{pieroth2025deep} use deep RL to compute equilibrium strategies in multistage auctions, and \citet{elmachtoub2022smart} establish the ``predict, then optimize'' framework integrating prediction with downstream decision-making, both demonstrating OR-native methodology enriching AI system design. Recent surveys of INFORMS Fellows identify bringing OR rigor to AI challenges as a priority research frontier \citep{kulkarni2025synergizing}. Our work contributes to this agenda by showing that Strong Data Processing Inequalities and minimax inspection design, classical OR tools, yield precise, actionable answers to questions the AI community has addressed primarily through empirical heuristics.

\section{Proofs for Section~\ref{sec:distinguish} (Signal Decay Bound)}\label{app:signal-decay}

\paragraph{Indexing convention.}
Throughout, the abstract chain has states $Z_0, Z_1, \ldots, Z_H$ connected by $H$ transition kernels $K_0, \ldots, K_{H-1}$, where $K_t$ maps $Z_t$ to $Z_{t+1}$.
Step-validity labels $r_t$ correspond to transition~$t$ (the transition producing $Z_{t+1}$ from $Z_t$).
When we say ``hypotheses differ at step $t$,'' we mean the transition kernel $K_{t-1}$ (producing $Z_t$) differs, so $H - t$ transitions remain before $Z_H$.

\begin{proof}{Proof of Proposition~\ref{prop:diversity}.}
The condition $K(z' \mid z) \geq p_{\min}$ for all $z, z'$ allows us to decompose $K$ into a mixing component and a residual. Let $\nu$ be the uniform distribution on $\mathcal{Z}$, so $\nu(z') = 1/|\mathcal{Z}|$ for all $z'$, and set $\alpha = |\mathcal{Z}| \cdot p_{\min}$. Then for any $z, z'$:
\[
K(z' \mid z) \geq p_{\min} = \alpha \cdot \nu(z'),
\]
so we can write:
\[
K(z' \mid z) = \alpha \cdot \nu(z') + (1 - \alpha) \cdot K_{\mathrm{res}}(z' \mid z),
\]
where $K_{\mathrm{res}}$ is a residual kernel. Under this decomposition, with probability $\alpha$ the next state is drawn from $\nu$ regardless of the current state, and with probability $1 - \alpha$ the next state is drawn from the residual kernel.

For any two distributions $P$ and $Q$ on $\mathcal{Z}$, the output distributions after passing through $K$ are:
\begin{align*}
PK &= \alpha \cdot \nu + (1 - \alpha) \cdot P K_{\mathrm{res}}, \\
QK &= \alpha \cdot \nu + (1 - \alpha) \cdot Q K_{\mathrm{res}}.
\end{align*}
The $\chi^2$-divergence between these mixtures satisfies:
\begin{align*}
\chi^2(PK \| QK)
&= \sum_{z'} \frac{((1-\alpha)(PK_{\mathrm{res}}(z') - QK_{\mathrm{res}}(z')))^2}{\alpha\,\nu(z') + (1-\alpha)\,QK_{\mathrm{res}}(z')} \\
&\leq (1-\alpha)^2 \sum_{z'} \frac{(PK_{\mathrm{res}}(z') - QK_{\mathrm{res}}(z'))^2}{(1-\alpha)\,QK_{\mathrm{res}}(z')} \\
&= (1-\alpha) \cdot \chi^2(PK_{\mathrm{res}} \| QK_{\mathrm{res}}) \\
&\leq (1-\alpha) \cdot \chi^2(P \| Q),
\end{align*}
where the first inequality drops the positive term $\alpha\,\nu(z')$ from the denominator, and the last inequality uses the standard data processing inequality for $K_{\mathrm{res}}$. Thus:
\[\eta_{\chi^2}(K) \leq 1 - \alpha = 1 - {|\mathcal{Z}|} \cdot p_{\min}.\]
\Halmos
\end{proof}

\begin{proof}{Proof of Proposition~\ref{prop:dobrushin}.}
Let $\eta_{\mathrm{TV}}(K)$ denote the total-variation (Dobrushin) contraction coefficient:
\[
\eta_{\mathrm{TV}}(K):=\sup_{P\neq Q}\frac{d_{\mathrm{TV}}(PK,QK)}{d_{\mathrm{TV}}(P,Q)}
=\max_{z,z'} d_{\mathrm{TV}}\bigl(K(\cdot\mid z),K(\cdot\mid z')\bigr).
\]
For probability vectors $p,q$ on a finite alphabet,
$d_{\mathrm{TV}}(p,q)=1-\sum_{y}\min\{p(y),q(y)\}$, hence $\eta_{\mathrm{TV}}(K)=1-\alpha(K)$ by the definition of $\alpha(K)$.
Finally, for every $f$-divergence (and in particular for $\chi^2$),
the corresponding contraction coefficient is upper bounded by $\eta_{\mathrm{TV}}(K)$ \citep[Theorem~1]{polyanskiy2017strong}.
Therefore,
\(
\eta_{\chi^2}(K)\le \eta_{\mathrm{TV}}(K)=1-\alpha(K),
\)
as claimed.
\Halmos
\end{proof}

\begin{proof}{Proof of Theorem~\ref{thm:signal_decay}.}
\textbf{Part 1: Divergence decay.}

Let $P^{(i)}_{Z_u}$ denote the distribution of $Z_u$ under hypothesis $H_i$ for $i \in \{0, 1\}$. Since the transition kernels are identical under both hypotheses, we have $P^{(i)}_{Z_{u+1}} = P^{(i)}_{Z_u} K_u$ for each $u \geq t$.

By the $\chi^2$ Strong Data Processing Inequality (Lemma~\ref{lem:sdpi}) and Assumption~\ref{ass:contractive}:
\[{
    \chi^2(P^{(0)}_{Z_{u+1}} \| P^{(1)}_{Z_{u+1}}) = \chi^2(P^{(0)}_{Z_u} K_u \| P^{(1)}_{Z_u} K_u) \leq \eta \cdot \chi^2(P^{(0)}_{Z_u} \| P^{(1)}_{Z_u}).
}\]

Iterating from step $t$ through steps $t+1, \ldots, H$ (applying kernels $K_t, K_{t+1}, \ldots, K_{H-1}$):
\[{
    \chi^2(P^{(0)}_{Z_H} \| P^{(1)}_{Z_H}) \leq \eta^{H-t} \cdot \chi^2(P^{(0)}_{Z_t} \| P^{(1)}_{Z_t}) = \eta^{H-t} \cdot \Delta^2.
}\]

Since $R = g(Z_H)$ is a deterministic function of the terminal state, the standard data processing inequality gives $\chi^2(P^{(0)}_R \| P^{(1)}_R) \leq \chi^2(P^{(0)}_{Z_H} \| P^{(1)}_{Z_H})$.

\textbf{Part 2: Sample complexity lower bound.}

Let $\delta_R := \chi^2(P^{(0)}_R \| P^{(1)}_R) \leq \eta^{H-t} \Delta^2$. For $n$ i.i.d.\ outcome observations, the $\chi^2$-divergence tensorizes:
\[{
    \chi^2(P^{(0)\otimes n}_R \| P^{(1)\otimes n}_R) = (1 + \delta_R)^n - 1.
}\]

We work in the regime $\delta_R \leq 1$. Using $(1+x)^n - 1 \leq e^{nx} - 1 \leq 2nx$ for $nx \leq 1$:
\[{
    \chi^2(P^{(0)\otimes n}_R \| P^{(1)\otimes n}_R) \leq 2n\delta_R \quad \text{when } n\delta_R \leq 1.
}\]

The total variation distance satisfies $d_{\mathrm{TV}}(P, Q) \leq \sqrt{\chi^2(P \| Q) / 2}$, so
\[{
    d_{\mathrm{TV}}(P^{(0)\otimes n}_R, P^{(1)\otimes n}_R) \leq \sqrt{n\delta_R}.
}\]

By the identity $\mathrm{err}^* = 1 - d_{\mathrm{TV}}$, if $n\delta_R < (1-\epsilon)^2$ then $d_{\mathrm{TV}} < 1-\epsilon$ and hence $\mathrm{err}^* > \epsilon$. By contrapositive, achieving total testing error at most $\epsilon$ requires $n\delta_R \geq (1-\epsilon)^2$, hence:
\[{
    n \geq \frac{(1-\epsilon)^2}{\delta_R} \geq \frac{(1-\epsilon)^2}{\eta^{H-t} \Delta^2}.
}\]

\textbf{Part 3: Exponential growth.}

When $\Delta^2 = \Theta(1)$, the bound becomes $n = \Omega(1/\eta^{H-t}) = \Omega((1/\eta)^{H-t})$. Since $\eta < 1$, we have $1/\eta > 1$, so sample complexity grows exponentially in $H - t$.
\Halmos
\end{proof}

\begin{remark}[Extension to Noisy Terminal Observations]
\label{rem:noisy-outcome}
Theorem~\ref{thm:signal_decay} assumes a deterministic terminal map $R = g(Z_H)$.
The result extends directly to noisy observations $R \sim G(\cdot \mid Z_H)$ for any Markov kernel $G: \mathcal{Z} \to \Delta(\{0,1\})$.
The standard data processing inequality gives $\chi^2(P^{(0)}_R \| P^{(1)}_R) \leq \eta_{\chi^2}(G) \cdot \chi^2(P^{(0)}_{Z_H} \| P^{(1)}_{Z_H})$, so the effective contraction product becomes $\eta_{\chi^2}(G) \cdot \prod_{u=t}^{H-1} \eta_u$, and the critical horizon shortens by $\ln(1/\eta_{\chi^2}(G)) / \ln(1/\eta)$ steps.
The deterministic case $R = g(Z_H)$ corresponds to $\eta_{\chi^2}(G) = 1$ when $g$ is injective.
\end{remark}

\begin{proof}{Proof of Corollary~\ref{cor:horizon}.}
The sample complexity lower bound in Theorem~\ref{thm:signal_decay} states that achieving error at most $\epsilon$ requires
\[
n \geq \frac{(1-\epsilon)^2}{\eta^{H-t}\Delta^2}.
\]
Rearranging for $\eta^{H-t}$:
\[
\eta^{H-t} \geq \frac{(1-\epsilon)^2}{n\Delta^2}.
\]
Taking logarithms of both sides (noting $\ln \eta < 0$ since $\eta < 1$, which reverses the inequality):
\[
(H-t) \ln \eta \geq \ln\left(\frac{(1-\epsilon)^2}{n\Delta^2}\right)
\quad \Longrightarrow \quad
H - t \leq \frac{\ln(n\Delta^2) - 2\ln(1-\epsilon)}{\ln(1/\eta)},
\]
which is the claimed bound.
\Halmos
\end{proof}

\begin{proof}{Proof of Proposition~\ref{prop:achievability}.}
For Bernoulli distributions with parameters $p_0$ and $p_1 = p_0 + \delta$, the $\chi^2$-divergence is:
\[
\chi^2(P^{(0)}_R \| P^{(1)}_R) = \frac{(p_0 - p_1)^2}{p_1} + \frac{((1-p_0) - (1-p_1))^2}{1-p_1} = \delta^2 \left(\frac{1}{p_1} + \frac{1}{1-p_1}\right) = \Theta(\delta^2) = \Theta(\eta^{H-t}\Delta^2),
\]
where the last step uses $\delta = \Theta(\eta^{(H-t)/2}\Delta)$. We assume $p_1 \in [\epsilon_0, 1-\epsilon_0]$ for some constant $\epsilon_0 > 0$; this mild regularity condition ensures the outcome is not deterministic under either hypothesis. The likelihood ratio test is known to achieve the minimax optimal rate, requiring $n = \Theta(1/\chi^2(P^{(0)}_R \| P^{(1)}_R)) = \Theta(1/(\eta^{H-t}\Delta^2))$ samples to distinguish the hypotheses with constant error probability.
\Halmos
\end{proof}

\begin{proof}{Proof of Corollary~\ref{cor:collapse}.}
Let $R^{1:n}$ denote $n$ i.i.d.\ outcome samples under hypothesis $H_i$, so $R^{1:n}\sim (P^{(i)}_R)^{\otimes n}$.
By Theorem~\ref{thm:signal_decay}, $\chi^2(P^{(0)}_R\|P^{(1)}_R)\le \eta^{H-t}\Delta^2$. Using $\chi^2$ tensorization for product measures,
\[
\chi^2\!\left((P^{(0)}_R)^{\otimes n}\,\middle\|\,(P^{(1)}_R)^{\otimes n}\right)
= \left(1+\chi^2(P^{(0)}_R\|P^{(1)}_R)\right)^n-1
\le (1+\eta^{H-t}\Delta^2)^n-1.
\]
By the standard bound $d_{\mathrm{TV}}(P,Q)\le \sqrt{\chi^2(P\|Q)/2}$,
\[
d_{\mathrm{TV}}\!\left((P^{(0)}_R)^{\otimes n},(P^{(1)}_R)^{\otimes n}\right)
\le \sqrt{\frac{(1+\eta^{H-t}\Delta^2)^n-1}{2}}.
\]
Finally, for any test $\psi$, Le Cam's bound yields
\[
\max_{i\in\{0,1\}}\Prob_i(\psi(R^{1:n})\neq i)\;\ge\;\tfrac12\left(1-d_{\mathrm{TV}}\!\left((P^{(0)}_R)^{\otimes n},(P^{(1)}_R)^{\otimes n}\right)\right),
\]
which yields the stated bound.
For the ``in particular'' statement, it suffices that
$\sqrt{((1+\eta^{H-t}\Delta^2)^n-1)/2}\le 1-2\epsilon$, i.e.,
$(1+\eta^{H-t}\Delta^2)^n-1\le 2(1-2\epsilon)^2$.
Using $\ln(1+x)\le x$ implies $(1+\eta^{H-t}\Delta^2)^n\le \exp(n\eta^{H-t}\Delta^2)$, so it is enough that
$n\eta^{H-t}\Delta^2\le \ln(1+2(1-2\epsilon)^2)$.
\Halmos
\end{proof}

\begin{example}[Markov-Chain Instance Achieving the Signal Decay Rate]
\label{ex:tight-chain}
We construct a Markov chain satisfying Definition~\ref{def:abstraction} and Assumption~\ref{ass:contractive} for which the Signal Decay Bound is achieved up to constants.
Let $|\mathcal{Z}|=2$ and define the abstract transition kernel $K = \bigl(\begin{smallmatrix} 1-q & q \\ q & 1-q \end{smallmatrix}\bigr)$ with $q \in (0,1/2)$, so $\eta_{\chi^2}(K) = (1-2q)^2$.
Set the terminal map $g(z) = z$ (state~$1$ is success, state~$0$ is failure).
Under $H_0$, let $Z_t = 0$ with probability~$1$; under $H_1$, let $Z_t = 1$ with probability $\delta$ and $0$ with probability $1-\delta$, giving $\chi^2(P^{(0)}_{Z_t} \| P^{(1)}_{Z_t}) = \delta^2/(1-\delta) = \Theta(\delta^2)$ for small~$\delta$.

After $H-t$ applications of~$K$, the outcome distributions satisfy
\[
\chi^2(P^{(0)}_R \| P^{(1)}_R) = \eta^{H-t} \cdot \chi^2(P^{(0)}_{Z_t} \| P^{(1)}_{Z_t}),
\]
with equality because the binary symmetric kernel with uniform reference measure achieves the $\chi^2$-SDPI bound \citep[Example~3]{polyanskiy2017strong}, and $g$ is bijective.  The likelihood ratio test then achieves the matching upper bound $n = \Theta(1/(\eta^{H-t}\Delta^2))$.
\end{example}

\begin{proposition}[Approximate Lumpability]
\label{prop:approx-lumpability}
Suppose the per-step discrepancy between true and lumpable dynamics is bounded by $\delta$ in total variation: $d_{\mathrm{TV}}(K_u(\cdot\mid z), \tilde K_u(\cdot\mid z)) \leq \delta$ for all $z \in \mathcal{Z}$ and $u = t, \ldots, H-1$. Assume the lumpable chain satisfies Assumption~\ref{ass:contractive} with contraction coefficient $\eta$. Then:
\[{
d_{\mathrm{TV}}(P^{(0)}_{R}, P^{(1)}_{R}) \leq \sqrt{\eta^{H-t}\Delta^2/2} + 2(H-t)\delta,
}\]

and, letting $\tau := \sqrt{\eta^{H-t}\Delta^2/2} + 2(H-t)\delta$, any test based on $n$ i.i.d.\ outcome samples satisfies
\[
\inf_{\psi}\max_{i\in\{0,1\}}\Prob_i(\psi(R^{1:n})\neq i)\;\ge\;\frac12\left(1-\min\{1,n\tau\}\right),
\]
so in particular achieving minimax error at most $\epsilon$ requires $n \ge (1-2\epsilon)/\tau$.
\end{proposition}

\begin{proof}{Proof of Proposition~\ref{prop:approx-lumpability}.}
Fix a hypothesis $i\in\{0,1\}$ and write $P_u$ and $\tilde P_u$ for the laws of $Z_u$ under the true and lumpable kernels, respectively.
By the triangle inequality and the fact that Markov kernels are non-expansive in total variation,
\[
d_{\mathrm{TV}}(P_{u+1},\tilde P_{u+1})
= d_{\mathrm{TV}}(P_u K_u,\tilde P_u \tilde K_u)
\le d_{\mathrm{TV}}(P_u K_u, P_u \tilde K_u) + d_{\mathrm{TV}}(P_u\tilde K_u,\tilde P_u\tilde K_u)
\le \delta + d_{\mathrm{TV}}(P_u,\tilde P_u),
\]
where the first term uses $d_{\mathrm{TV}}(\mu K, \mu \tilde K) \leq \sup_z d_{\mathrm{TV}}(K(\cdot \mid z), \tilde K(\cdot \mid z)) \leq \delta$ for any distribution $\mu$, and the second uses non-expansiveness of $\tilde K_u$.
Iterating from $u=t$ (where $P_t=\tilde P_t$) yields
$d_{\mathrm{TV}}(P_{Z_H}^{(i)},\tilde P_{Z_H}^{(i)})\le (H-t)\delta$.
The same bound transfers to any terminal observation $R$ by data processing.

For the main inequality, apply the triangle inequality:
\[
d_{\mathrm{TV}}(P^{(0)}_{R},P^{(1)}_{R})
\le d_{\mathrm{TV}}(P^{(0)}_{R},\tilde P^{(0)}_{R})
+ d_{\mathrm{TV}}(\tilde P^{(0)}_{R},\tilde P^{(1)}_{R})
+ d_{\mathrm{TV}}(\tilde P^{(1)}_{R},P^{(1)}_{R}).
\]
The first and third terms are bounded by $(H-t)\delta$. For the middle term, the inequality $d_{\mathrm{TV}}(P,Q)\le \sqrt{\chi^2(P\|Q)/2}$ combined with Theorem~\ref{thm:signal_decay} for the lumpable chain yields $d_{\mathrm{TV}}(\tilde P^{(0)}_{R},\tilde P^{(1)}_{R})\le \sqrt{\eta^{H-t}\Delta^2/2}$.

For $n$ i.i.d.\ outcome samples, write $P_n := (P^{(0)}_R)^{\otimes n}$ and $Q_n := (P^{(1)}_R)^{\otimes n}$.
A standard telescoping argument gives the subadditivity bound
\[
d_{\mathrm{TV}}(P_n,Q_n)\le n\,d_{\mathrm{TV}}(P^{(0)}_R,P^{(1)}_R)\le n\tau,
\]
and therefore $d_{\mathrm{TV}}(P_n,Q_n)\le \min\{1,n\tau\}$.
By Le Cam's two-point method, for any test $\psi$,
\[
\max_{i\in\{0,1\}}\Prob_i(\psi(R^{1:n})\neq i)\;\ge\;\tfrac12\left(1-d_{\mathrm{TV}}(P_n,Q_n)\right).
\]
Taking $\inf_{\psi}$ on the left-hand side yields
\[
\inf_{\psi}\max_{i\in\{0,1\}}\Prob_i(\psi(R^{1:n})\neq i)\;\ge\;\tfrac12\left(1-d_{\mathrm{TV}}(P_n,Q_n)\right)
\;\ge\;\tfrac12\left(1-\min\{1,n\tau\}\right),
\]
which yields the claimed lower bound. In particular, if $n\tau < 1-2\epsilon$, then every test has minimax error $>\epsilon$.
\Halmos
\end{proof}

\section{Proofs for Section~\ref{sec:width} (Width Limits)}\label{app:width}

\begin{proof}{Proof of Proposition~\ref{prop:width}.}
\textbf{Part 1:} Each $R_j$ is a Bernoulli random variable with $\mathbb{E}[R_j \mid s] = V^\pi(s)$. By linearity of expectation:
\[{
    \mathbb{E}[\hat{V}_W(s) \mid s] = \frac{1}{W} \sum_{j=1}^W \mathbb{E}[R_j \mid s] = V^\pi(s).
}\]

\textbf{Part 2:} Since the $R_j$ are conditionally independent given $s$, each with variance $\mathrm{Var}(R_j \mid s) = V^\pi(s)(1 - V^\pi(s))$:
\[{
    \mathrm{Var}(\hat{V}_W(s) \mid s) = \frac{1}{W^2} \sum_{j=1}^W \mathrm{Var}(R_j \mid s) = \frac{V^\pi(s)(1 - V^\pi(s))}{W}.
}\]

\textbf{Part 3:} By Hoeffding's inequality applied to $W$ independent random variables bounded in $[0, 1]$:
\[{
    \Prob(|\hat{V}_W(s) - V^\pi(s)| \geq \epsilon \mid s) \leq 2\exp(-2W\epsilon^2).
}\]

Setting $\delta = 2\exp(-2W\epsilon^2)$ and solving for $\epsilon$ yields the stated bound.
\Halmos
\end{proof}

\begin{proof}{Proof of Theorem~\ref{thm:correlated}.}
The variance of the sum of correlated random variables is:
\[{
    \mathrm{Var}\left(\sum_{j=1}^W R_j \mid s\right) = \sum_{j=1}^W \mathrm{Var}(R_j \mid s) + 2\sum_{i < j} \mathrm{Cov}(R_i, R_j \mid s).
}\]

With $\mathrm{Var}(R_j \mid s) = \sigma^2(s)$ and $\mathrm{Cov}(R_i, R_j \mid s) = \rho \cdot \sigma^2(s)$ for $i \neq j$:
\[{
    \mathrm{Var}\left(\sum_{j=1}^W R_j \mid s\right) = W \sigma^2(s) + W(W-1) \rho \sigma^2(s) = \sigma^2(s) \cdot W(1 + (W-1)\rho).
}\]

Dividing by $W^2$ gives the stated variance formula. The effective width satisfies $\sigma^2(s)/W_{\mathrm{eff}} = \mathrm{Var}(\hat{V}_W(s) \mid s)$, yielding $W_{\mathrm{eff}} = W/(1 + (W-1)\rho)$. When $\rho = 0$, this gives $W_{\mathrm{eff}} = W$. When $\rho > 0$, as $W \to \infty$ this approaches $1/\rho$.
\Halmos
\end{proof}

\begin{proof}{Proof of Corollary~\ref{cor:width-horizon}.}
Substitute $N_{\mathrm{eff}} = n \cdot W_{\mathrm{eff}}$ into Corollary~\ref{cor:horizon}.
\Halmos
\end{proof}

\section{Proofs for Section~\ref{sec:bellman} (Objective Mismatch)}\label{app:bellman}

\begin{proof}{Proof of Proposition~\ref{prop:curriculum}.}
By linearity, $\nabla_\theta J_{\lambda} = (1-\lambda)\nabla_\theta J_{\mathrm{add}} + \lambda \nabla_\theta J_{\mathrm{mult}}$. The non-negative inner product ensures $\|\nabla_\theta J_{\lambda}\| \geq (1-\lambda) \|\nabla_\theta J_{\mathrm{add}}\|$. Setting $(1-\lambda) \geq c$ yields the result.
\Halmos
\end{proof}

\section{Proofs for Section~\ref{sec:inspection-design} (Optimal Inspection Design)}\label{app:inspection}

\begin{proof}{Proof of Proposition~\ref{prop:inspection-decay}.}
By the $\chi^2$-SDPI (Lemma~\ref{lem:sdpi}), each transition contracts divergence: $\chi^2(P^{(0)}_{Z_{j+1}} \| P^{(1)}_{Z_{j+1}}) \leq \eta_j \cdot \chi^2(P^{(0)}_{Z_j} \| P^{(1)}_{Z_j})$. Iterating from $j = t$ to $u-1$ yields $\chi^2(P^{(0)}_{Z_u} \| P^{(1)}_{Z_u}) \leq \Attn(t \to u) \Delta^2$. Since $Y_u = g_u(Z_u)$ is a function of $Z_u$, the data processing inequality gives $\chi^2(P^{(0)}_{Y_u} \| P^{(1)}_{Y_u}) \leq \chi^2(P^{(0)}_{Z_u} \| P^{(1)}_{Z_u})$.
\Halmos
\end{proof}

\begin{remark}[Imperfect Inspections]
\label{rem:imperfect-inspection}
Definition~\ref{def:inspection-observation} models inspections as noiseless functions of the abstract state.
In practice, inspections may have false positives and negatives: the observation becomes $Y_u \sim G_u(\cdot \mid Z_u)$ for a noisy channel $G_u$ with $\chi^2$-contraction coefficient $\eta_{\chi^2}(G_u) \leq 1$.
Proposition~\ref{prop:inspection-decay} then generalizes to $\chi^2(P^{(0)}_{Y_u} \| P^{(1)}_{Y_u}) \leq \eta_{\chi^2}(G_u) \cdot \Attn(t \to u) \cdot \Delta^2$.
In the information-distance framework of Section~\ref{subsec:heterogeneous-opt}, this adds a penalty $\ln(1/\eta_{\chi^2}(G_u))$ at each inspection, effectively shortening the feasible segment length.
Higher-fidelity inspections (larger $\eta_{\chi^2}(G_u)$, closer to~$1$) incur a smaller penalty.
\end{remark}

\begin{proof}{Proof of Corollary~\ref{cor:inspection-sample-lb}.}
Apply the testing lower bound from Theorem~\ref{thm:signal_decay} with the effective horizon replaced by $u - t$.
\Halmos
\end{proof}

\begin{proof}{Proof of Proposition~\ref{prop:largest-gap-lb}.}
Let $t^\star$ be the first step in the largest segment: $t^\star = t_{i^\star}$ where $i^\star \in \arg\max_i (t_{i+1} - t_i)$. The next inspection occurs at $u^\star = t_{i^\star+1}$, so $u^\star - t^\star = L(\mathcal{T})$. Applying Corollary~\ref{cor:inspection-sample-lb} with $u = u^\star$ yields the bound.
\Halmos
\end{proof}

\begin{proof}{Proof of Theorem~\ref{thm:uniform-optimal}.}
Any schedule with $m$ intermediate inspections partitions the horizon into $m+1$ segments with integer lengths $\ell_0,\ldots,\ell_m$ satisfying $\sum_{i=0}^{m}\ell_i = H$. Hence $\max_i \ell_i \ge \lceil H/(m+1)\rceil$. For the near-uniform schedule $t_i=\lfloor iH/(m+1)\rfloor$, each segment length $t_{i+1}-t_i$ is either $\lfloor H/(m+1)\rfloor$ or $\lceil H/(m+1)\rceil$, so $\max_i (t_{i+1}-t_i)=\lceil H/(m+1)\rceil$, achieving the lower bound.
\Halmos
\end{proof}

\begin{proof}{Proof of Corollary~\ref{cor:minimax-lb-m}.}
Under minimax-optimal placement with $m$ inspections, Theorem~\ref{thm:uniform-optimal} gives $L(\mathcal{T})=\lceil H/(m+1)\rceil$. Substituting this into Proposition~\ref{prop:largest-gap-lb} yields the stated bound.
\Halmos
\end{proof}

\begin{proof}{Proof of Corollary~\ref{cor:min-inspections}.}
Feasibility requires $L(\mathcal{T}) \leq H_{\mathrm{crit}}(\epsilon)$. With optimal spacing, $L(\mathcal{T}) = \lceil H/(m+1) \rceil$, so $(m+1) \geq H / H_{\mathrm{crit}}(\epsilon)$.
\Halmos
\end{proof}

\begin{proof}{Proof of Lemma~\ref{lem:segment-feasible}.}
Corollary~\ref{cor:inspection-sample-lb} implies that step $t_i$ is $(n,\epsilon)$-distinguishable using the inspection at $t_{i+1}$ only if
\[
n \geq \frac{(1-\epsilon)^2}{\Attn(t_i \to t_{i+1}) \Delta^2},
\]
equivalently $W([t_i, t_{i+1})) \leq \ln(n\Delta^2) - 2\ln(1-\epsilon)$. If the inequality holds, then for any $t\in\{t_i,\ldots,t_{i+1}-1\}$ we have $\Attn(t \to t_{i+1}) \geq \Attn(t_i \to t_{i+1})$ (the product runs over fewer factors in $(0,1]$), so the required sample size for step $t$ is no larger than for step $t_i$. Hence every step in the segment is $(n,\epsilon)$-distinguishable. Conversely, if step $t_i$ is $(n,\epsilon)$-distinguishable using the inspection at $t_{i+1}$, then the displayed inequality must hold by Corollary~\ref{cor:inspection-sample-lb}.
\Halmos
\end{proof}

\begin{proof}{Proof of Theorem~\ref{thm:greedy-opt}.}
Let the greedy schedule start at $t_0^g=0$ and choose each $t_{i+1}^g$ as the largest index such that $W([t_i^g,t_{i+1}^g))\le \Gamma$. Consider any feasible schedule with inspection times $0=t_0<t_1<\cdots$. Because it is feasible, it must satisfy $W([0,t_1))\le \Gamma$, hence by maximality of $t_1^g$ we have $t_1\le t_1^g$. Now apply the same argument to the residual instance starting at $t_1^g$: any feasible continuation from $t_1$ cannot place its next inspection beyond the greedy choice from $t_1^g$ without violating the constraint $W(\cdot)\le \Gamma$. Formally, if $t_j \leq t_j^g$ for some $j$, then any feasible $t_{j+1}$ from $t_j$ satisfies $W([t_j, t_{j+1})) \leq \Gamma$, hence $W([t_j^g, t_{j+1})) \leq W([t_j, t_{j+1})) \leq \Gamma$ (since removing leading terms with $w_t \geq 0$ cannot increase the sum), so $t_{j+1} \leq t_{j+1}^g$ by maximality of $t_{j+1}^g$. Iterating, after $k$ greedy inspections the greedy schedule covers at least as much of the horizon as any feasible schedule with $k$ inspections. Therefore no feasible schedule can reach $H$ with fewer inspections than greedy, so greedy minimizes the number of inspections.
\Halmos
\end{proof}

\begin{proof}{Proof of Proposition~\ref{prop:budget-lb}.}
Corollary~\ref{cor:minimax-lb-m} requires $n \geq (1-\epsilon)^2 / (\eta^{\lceil H/(m+1)\rceil}\Delta^2)$. Multiplying by the per-trajectory cost gives the result.
\Halmos
\end{proof}

\begin{proof}{Proof of Corollary~\ref{cor:poly-density}.}
Feasibility with $n \leq H^p$ requires $H^p \gtrsim \eta^{-\lceil H/(m+1)\rceil}$. Taking logarithms: $p \ln H \gtrsim \lceil H/(m+1)\rceil \ln(1/\eta)$, so $m + 1 \gtrsim (H / \ln H) \cdot (\ln(1/\eta) / p) = \Omega(H / \log H)$.
\Halmos
\end{proof}

\section{Practical Examples}
\label{app:inspection-examples}

\paragraph{Interpreting the initial divergence $\Delta^2$.}
The parameter $\Delta^2 = \chi^2(P^{(0)}_{Z_t} \| P^{(1)}_{Z_t})$ measures how different the abstract-state distributions are under ``step correct'' ($H_0$) versus ``step erroneous'' ($H_1$).
In manufacturing, $\Delta^2$ reflects the separation between the grade distribution produced by a correctly operating stage versus a malfunctioning one; it can be estimated from pilot audits by running the stage under controlled conditions (normal vs.\ degraded), binning outputs into quality grades, and computing the empirical $\chi^2$-divergence.
Larger $\Delta^2$ means defects are easier to detect locally, extending the critical horizon; small $\Delta^2$ (subtle defects) shortens it.
When $\Delta^2$ is unknown, the critical horizon formula can be evaluated at a conservative lower bound to yield a worst-case inspection schedule.

\begin{example}[Manufacturing Grade Transitions]
\label{ex:manufacturing-kernel}
Consider a three-grade abstraction $Z_t \in \{\mathrm{A},\mathrm{B},\mathrm{C}\}$ where $\mathrm{A}$ denotes ``within specification,'' $\mathrm{B}$ denotes ``marginal/reworkable,'' and $\mathrm{C}$ denotes ``defective,'' in a setting with rework and defect masking. A downstream step such as polishing followed by retest can partially repair marginal units while occasionally degrading good ones. A representative grade-to-grade kernel is
\[
K =
\begin{pmatrix}
0.85 & 0.14 & 0.01 \\
0.55 & 0.35 & 0.10 \\
0.20 & 0.30 & 0.50
\end{pmatrix},
\]
where rows index the incoming grade and columns the outgoing grade.
The minimum pairwise row overlap is achieved by the $\mathrm{A}$ and $\mathrm{C}$ rows:
\[
\begin{aligned}
\sum_{z'}\min\bigl\{K(z'\mid \mathrm{A}),\,K(z'\mid \mathrm{C})\bigr\}
&= \min\{0.85,0.20\}+\min\{0.14,0.30\}+\min\{0.01,0.50\}\\
&= 0.20+0.14+0.01 = 0.35,
\end{aligned}
\]
so $\alpha(K)=0.35$ and Proposition~\ref{prop:dobrushin} gives $\eta_{\chi^2}(K)\le 1-\alpha(K)=0.65$. Even when the underlying physical transformation is repeatable, coarse grading and partial defect masking produce strict per-stage contraction on the abstract state space.
\end{example}

\begin{example}[Semiconductor Fabrication]
\label{ex:semiconductor}
A semiconductor process has $H = 50$ stages with $\eta = 0.85$. Terminal testing costs \$10/wafer; intermediate metrology costs \$50/inspection. With $n = 10{,}000$ wafers, $\Delta^2 = 0.2$, and target error $\epsilon=0.1$:
\[
 H_{\mathrm{crit}}(\epsilon) = \frac{\ln(2000) - 2\ln(0.9)}{\ln(1/0.85)} \approx 48.1 < 50.
\]
Since $H > H_{\mathrm{crit}}(\epsilon)$, outcome-only evaluation is insufficient. Corollary~\ref{cor:min-inspections} gives $m \geq \lceil 50/48.1 \rceil - 1 = 1$. One inspection at the midpoint ($t_1 = 25$) reduces the maximal gap to 25, bringing all stages within the critical horizon. The inspection adds \$50/wafer but makes the study feasible; without it, the required sample size would exceed the budget.
\end{example}

\begin{example}[Service Operations]
\label{ex:service}
A customer journey has $H = 50$ touchpoints. Early touchpoints (awareness, consideration) have high variability ($\eta_t = 0.6$ for $t \leq 10$), while later touchpoints are more deterministic ($\eta_t = 0.95$ for $t > 10$). With $n = 1{,}000$ customers, $\Delta^2 = 0.3$, and $\epsilon = 0.1$, the feasibility threshold is $\Gamma = \ln(300) - 2\ln(0.9) \approx 5.9$.

The 11 high-contraction steps contribute $W([0, 11)) = 11 \times \ln(1/0.6) \approx 5.6$, nearly exhausting the budget. The greedy algorithm extends through 5 additional low-contraction steps before exceeding $\Gamma$, placing an inspection at $t_1 = 16$ (since $W([0, 16)) \approx 5.88 \leq \Gamma \approx 5.91$ but $W([0, 17)) \approx 5.93 > \Gamma$). The remaining 34 low-contraction steps contribute only $W([16, 50)) = 34 \times \ln(1/0.95) \approx 1.7 < \Gamma$, fitting in one segment. The optimal schedule has one inspection, placed where contraction is strongest, demonstrating that non-uniform placement arises naturally when contraction varies.
\end{example}

\section{Budgeting Tradeoffs and Inspection Design Procedure}
\label{app:budgeting}

Section~\ref{sec:inspection-design} characterizes inspection placement in information units. This appendix addresses two operational questions: how to trade off sample size against inspection frequency, and when to prefer additional inspections over additional rollouts.

\paragraph{Sample--inspection tradeoff.}
Let $c_{\mathrm{out}}$ denote the per-trajectory cost of terminal evaluation and $c_{\mathrm{insp}}$ the per-trajectory cost of each intermediate inspection. Collecting $n$ trajectories with $m$ inspections each costs $B = n(c_{\mathrm{out}} + m \cdot c_{\mathrm{insp}})$. Proposition~\ref{prop:budget-lb} reveals that increasing $m$ reduces the sample requirement exponentially but raises per-trajectory cost linearly. A simple rule is to increase $m$ until the maximal gap $\lceil H/(m+1)\rceil$ falls below $H_{\mathrm{crit}}(\epsilon)$, at which point further inspections yield negligible sample savings. Corollary~\ref{cor:poly-density} implies that inspection effort must grow nearly linearly with process depth to avoid exponential sample requirements.

\paragraph{Width versus inspection.}
A practitioner choosing between increasing width $W$ (more rollouts or replicates) and adding an inspection faces the following comparison.
By Corollary~\ref{cor:width-horizon}, doubling $W_{\mathrm{eff}}$ extends the critical horizon by $\ln 2 / \ln(1/\eta)$ steps, but correlation caps $W_{\mathrm{eff}}$ at $1/\rho$, so no amount of additional rollouts can extend the horizon beyond $\ln(n \Delta^2 / \rho) / \ln(1/\eta)$.
By contrast, a single well-placed inspection halves the maximal gap (Theorem~\ref{thm:uniform-optimal}), extending the effective horizon by a factor of~$2$.
For processes where $H > \ln(n \Delta^2 / \rho) / \ln(1/\eta)$, width alone is provably insufficient and intermediate inspection is the only path to feasibility.
In practice, the marginal value of one additional inspection exceeds that of doubling $W$ whenever the current maximal gap $L(\mathcal{T})$ exceeds $H_{\mathrm{crit}}(\epsilon)$.

\paragraph{Summary procedure.}
The following procedure synthesizes the results of Section~\ref{sec:inspection-design} and the tradeoffs above.
\vspace{12pt}
\begin{center}
\fbox{\parbox{0.92\textwidth}{
\vspace{6pt}
\textbf{Inspection Design Procedure}
\begin{enumerate}[leftmargin=*,itemsep=2pt,topsep=4pt]
\item \emph{Estimate contraction:} Bound the contraction coefficient $\eta$ (or $\{\eta_t\}$) using A/B comparisons, cohort tracking, or variance decay measurements.
\item \emph{Compute feasibility threshold:} $\Gamma(n,\Delta^2,\epsilon)=\ln(n\Delta^2)-2\ln(1-\epsilon)$.
Equivalently, under homogeneous contraction compute $H_{\mathrm{crit}}(\epsilon)=\Gamma/\ln(1/\eta)$.
\item \emph{Choose inspections for feasibility:} Under homogeneous contraction, choose the smallest $m$ satisfying $m \geq \lceil H / H_{\mathrm{crit}}(\epsilon)\rceil - 1$.
Under heterogeneous contraction, choose the smallest $m$ produced by the greedy rule in Theorem~\ref{thm:greedy-opt}.
\item \emph{Place inspections:} Uniformly at $t_i = \lfloor iH/(m+1) \rfloor$ under homogeneous contraction (Theorem~\ref{thm:uniform-optimal}), or via greedy information-distance placement under heterogeneous contraction (Theorem~\ref{thm:greedy-opt}).
\item \emph{Budget the study:} Set $n$ and $m$ to satisfy $B=n(c_{\mathrm{out}}+m c_{\mathrm{insp}})$ and compare against the lower bound in Proposition~\ref{prop:budget-lb}.
\end{enumerate}
\vspace{6pt}
}}
\end{center}
\vspace{12pt}
\paragraph{Practical implications.}
Three insights emerge. First, intermediate inspection is often an investment in \emph{feasibility}, not merely precision: a single well-placed checkpoint can make an otherwise impossible attribution problem tractable (Example~\ref{ex:semiconductor}). Second, uniform spacing provides a robust default when contraction rates are uncertain (Theorem~\ref{thm:uniform-optimal}); when stage-specific estimates are available, the greedy algorithm (Theorem~\ref{thm:greedy-opt}) concentrates inspections where contraction is strongest. Third, inspection effort must scale with process depth: organizations extending their value chains cannot hold inspection budgets constant (Corollary~\ref{cor:poly-density}).

\section{Numerical Validation: Full Experimental Details}
\label{app:numerical-details}

We validate five core theoretical predictions using synthetic Markov chains with controlled parameters. For effective width and objective mismatch, we additionally demonstrate the predicted phenomena on LLM reasoning chains from GSM8K. Synthetic experiments allow exact control of the contraction coefficient $\eta$, providing clean verification of the bounds; LLM experiments confirm that the phenomena occur in practice. The experiments cover signal decay (Section~\ref{subsec:num-signal}), effective width under correlation (Section~\ref{subsec:num-width}), optimal inspection design (Section~\ref{subsec:num-inspection}), the critical horizon phase transition (Section~\ref{subsec:num-horizon}), and the objective mismatch between additive and multiplicative rewards (Section~\ref{subsec:num-mismatch}). The code and data to support the numerical experiments in this paper will be made publicly available upon acceptance.

\subsection{Signal Decay Bound}
\label{subsec:num-signal}

Theorem~\ref{thm:signal_decay} establishes that $\chi^2$-divergence between outcome distributions decays exponentially with distance from outcome:
\[
\chi^2(P_{S_H} \| Q_{S_H}) \;\leq\; \eta^{H-t} \cdot \chi^2(P_{S_t} \| Q_{S_t}).
\]
To test this bound, we construct synthetic Markov chains with 10 abstract states and exact contraction coefficient $\eta \in \{0.7, 0.8, 0.9, 0.95\}$, using the kernel $K = \sqrt{\eta}\, I + (1-\sqrt{\eta})\,(1/10)\,\mathbf{1}\mathbf{1}^\top$, a convex combination of the identity and the uniform-mixing matrix whose $\chi^2$-contraction coefficient is exactly $\eta$. For each step $t$, we set $P_t = \delta_0$ (point mass on state~0) and $Q_t = \pi$ (the uniform stationary distribution), giving $\chi^2(P_t \| Q_t) = 9$. We propagate both distributions to the terminal state through the contractive kernel and measure the resulting $\chi^2$-divergence at the state distribution level.

Figure~\ref{fig:signal_decay} plots the measured divergence against distance to outcome $(H-t)$ on a logarithmic scale for $H = 40$, showing all four values of~$\eta$. The dashed lines show the theoretical prediction $\eta^{H-t} \cdot \Delta^2$ with $\Delta^2 = 9$. Because the reference distribution $Q_t$ equals the stationary distribution~$\pi$, the bound is achieved with equality: measured divergences lie exactly on the theoretical lines. The measurements follow perfect linear trends on the log scale with distinct slopes for each~$\eta$, confirming exponential decay at the predicted rates. Results for other horizons ($H \in \{10, 20\}$) are identical in pattern.

\begin{figure}
\FIGURE
{\includegraphics[width=0.57\textwidth]{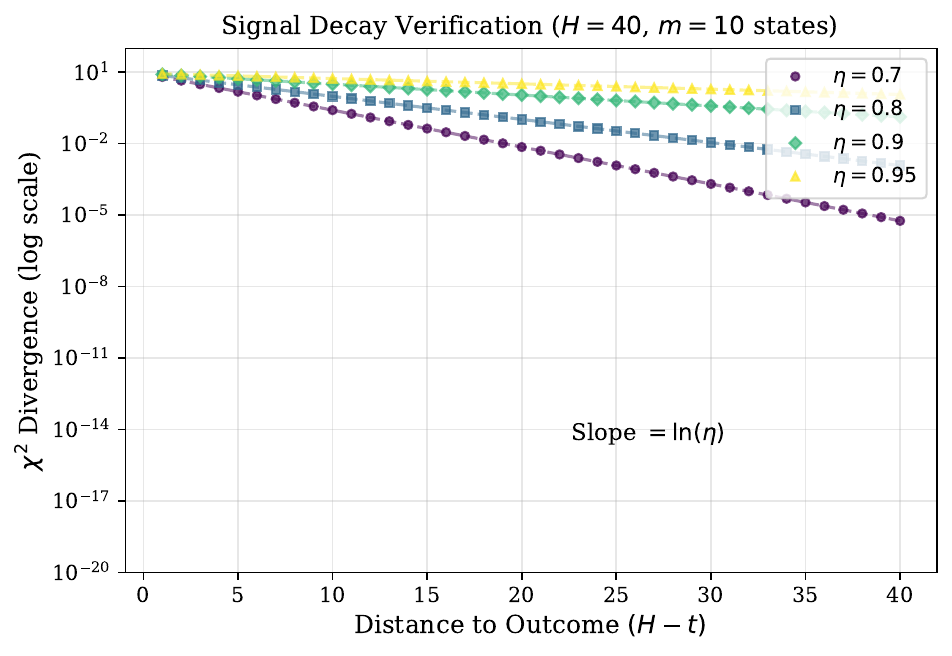}}
{Validation of the Signal Decay Bound (Theorem~\ref{thm:signal_decay}). Measured $\chi^2$-divergences (dots) match the theoretical prediction $\eta^{H-t} \cdot \Delta^2$ (dashed lines) exactly for all four contraction rates, confirming exponential decay with rate~$\eta$. Setup: 10-state Markov chains with $H = 40$ and $\Delta^2 = 9$.\label{fig:signal_decay}}
{}
\end{figure}

\subsection{Effective Width Under Correlation}
\label{subsec:num-width}

Theorem~\ref{thm:correlated} shows that positively correlated rollouts yield diminished variance reduction:
\[
W_{\mathrm{eff}} \;=\; \frac{W}{1 + (W-1)\rho},
\]
where $\rho$ is the pairwise correlation between rollout outcomes. As $W \to \infty$, effective width saturates at $1/\rho$, meaning that beyond a certain point, additional rollouts provide negligible benefit.

We test this prediction in two settings. In synthetic experiments, we induce correlation $\rho = 0.15$ by adding a shared noise component across rollouts, then measure variance reduction at $W \in \{1, 4, 16, 64, 256\}$. In LLM experiments, we generate $W \in \{1, 4, 16\}$ rollouts per GSM8K problem and measure the empirical pairwise correlation ($\bar\rho \approx 0.49$, averaged across $W \in \{4,16\}$) and variance reduction.

Figure~\ref{fig:width} shows the results. In both panels, measured effective width (red dots) closely tracks the theoretical prediction (blue curve). The synthetic experiment confirms saturation: at $W = 256$ rollouts, effective width reaches only $6.6$, compared to the theoretical maximum of $1/\rho = 6.7$. The LLM experiment shows even stronger saturation due to higher correlation: $W = 16$ nominal rollouts yield effective width of only $2.2$.

\begin{figure}
\FIGURE
{\includegraphics[width=0.95\textwidth]{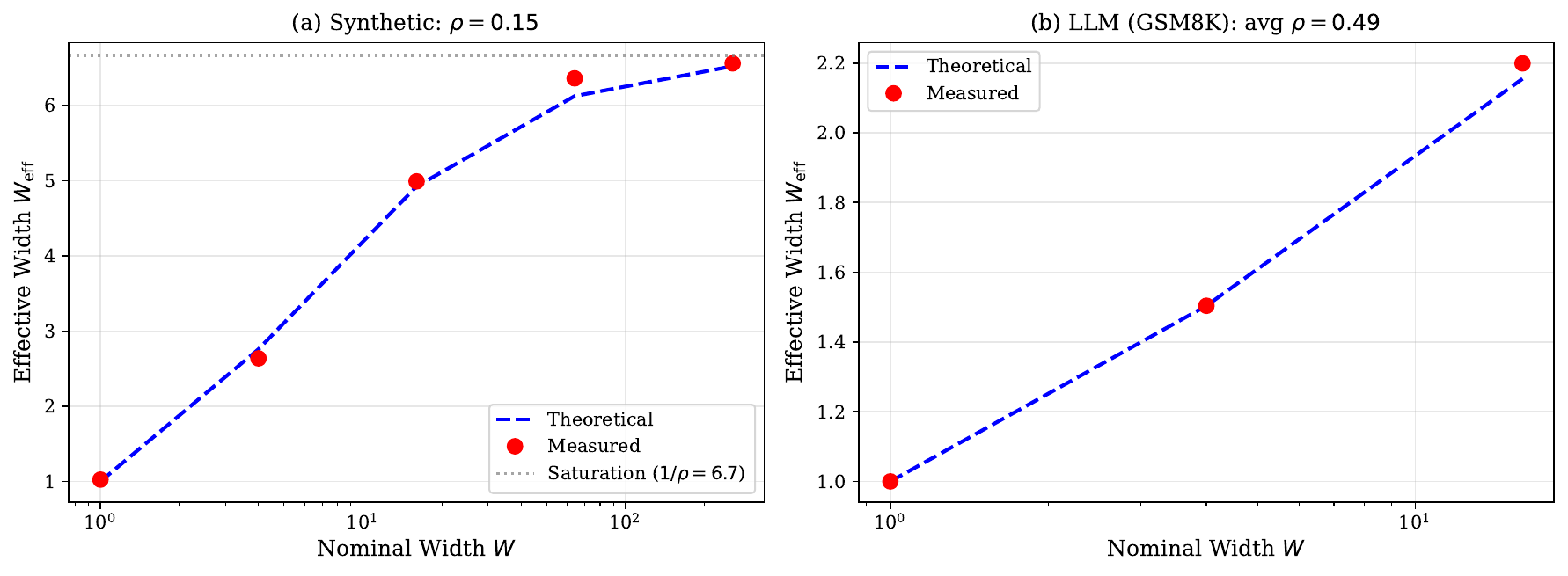}}
{Validation of the Effective Width Formula (Theorem~\ref{thm:correlated}). Measured effective width (dots) matches the theoretical prediction $W/(1+(W-1)\rho)$ (curve) and saturates at $1/\rho$. (a) Synthetic with $\rho = 0.15$. (b) LLM with $\rho \approx 0.49$.\label{fig:width}}
{}
\end{figure}

\subsection{Optimal Inspection Design}
\label{subsec:num-inspection}

Theorem~\ref{thm:uniform-optimal} establishes that under homogeneous contraction, uniformly spaced checkpoints minimize the worst-case attribution error across all steps. The intuition is that any gap between checkpoints allows signal to decay; uniform spacing ensures no gap is larger than necessary.

We test this with $H = 20$ steps, contraction coefficient $\eta = 0.9$, and $m = 3$ intermediate checkpoints, comparing four schedules: uniform placement at $\{5, 10, 15\}$, front-loaded at $\{2, 4, 6\}$, back-loaded at $\{14, 16, 18\}$, and random at $\{2, 13, 14\}$. For each schedule, we compute the worst-case attribution error across all 20 steps.

Figure~\ref{fig:inspection} shows the results. Uniform spacing achieves worst-case error of 0.41, while the alternatives achieve 0.69 (random) to 0.77 (front- and back-loaded). This represents a 40--47\% reduction in worst-case error, confirming that uniform spacing is not merely a reasonable heuristic but the optimal choice under the conditions of the theorem.

\begin{figure}
\FIGURE
{\includegraphics[width=0.6175\textwidth]{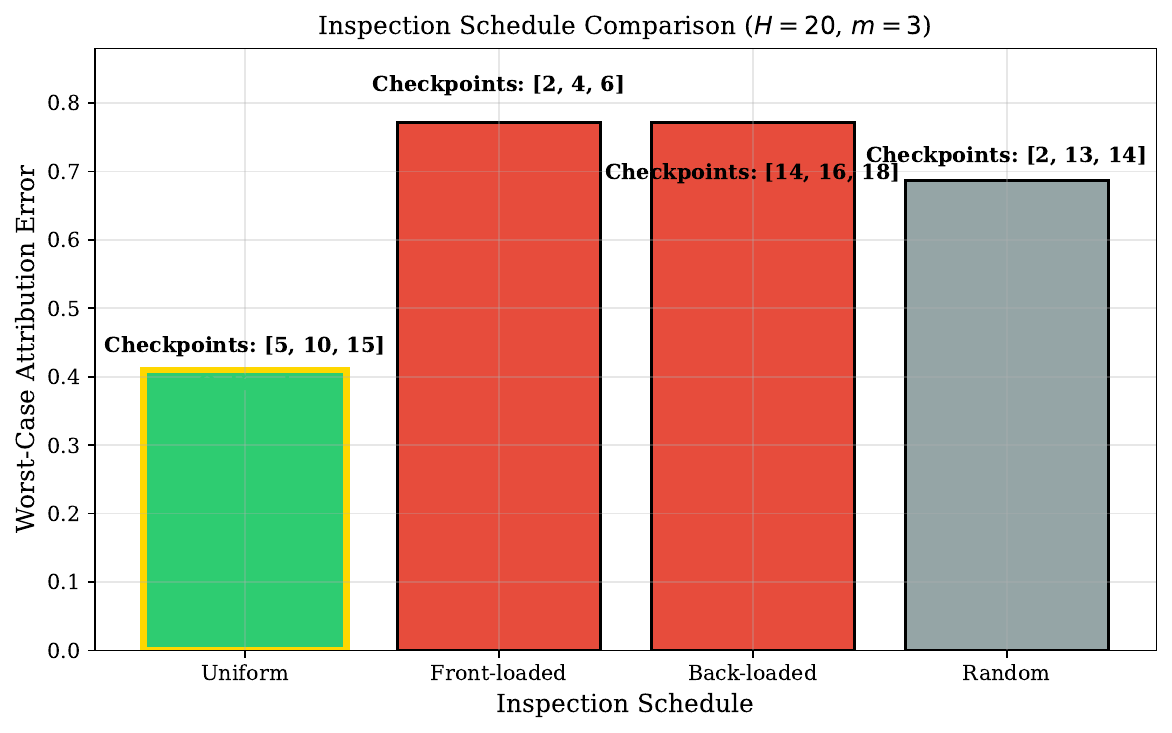}}
{Validation of Optimal Inspection Design (Theorem~\ref{thm:uniform-optimal}). Uniform checkpoint spacing (green) achieves 40--47\% lower worst-case error than front-loaded, back-loaded, or random schedules. Checkpoint positions shown above each bar.\label{fig:inspection}}
{}
\end{figure}

\subsection{Critical Horizon Phase Transition}
\label{subsec:num-horizon}

Corollary~\ref{cor:horizon} predicts a sharp phase transition: steps within $H_{\mathrm{crit}} \approx \ln(n\Delta^2)/\ln(1/\eta)$ of the outcome can be reliably attributed, while steps beyond this horizon cannot. We test this prediction by measuring attribution accuracy as a function of distance from outcome.

We use $H = 40$ steps, 10 abstract states, $n = 1{,}000$ samples per hypothesis, and $\Delta^2 = 9$. We test two contraction rates: $\eta = 0.7$ (predicted $H_{\mathrm{crit}} \approx 26$) and $\eta = 0.8$ (predicted $H_{\mathrm{crit}} \approx 41$). For each step position, we generate outcome samples under two initial distributions and test whether the step's identity can be recovered.

Figure~\ref{fig:critical_horizon} shows the results. Attribution accuracy for $\eta = 0.7$ drops from 89\% at distance~1 to chance level (50\%) around distance~23, slightly before the theoretical $H_{\mathrm{crit}} \approx 26$. For $\eta = 0.8$, the slower decay rate places $H_{\mathrm{crit}} \approx 41$ beyond the maximum horizon $H = 40$, so accuracy remains above chance throughout but declines steadily. The critical horizon formula provides a tight upper bound on the reach of outcome-based credit assignment: accuracy degrades continuously with distance and reaches chance level at or before $H_{\mathrm{crit}}$, confirming that steps beyond the critical horizon cannot be reliably evaluated. The gap between the empirical transition and the theoretical bound reflects the difference between the simple classifier used here and the information-theoretically optimal test.

\begin{figure}
\FIGURE
{\includegraphics[width=0.6175\textwidth]{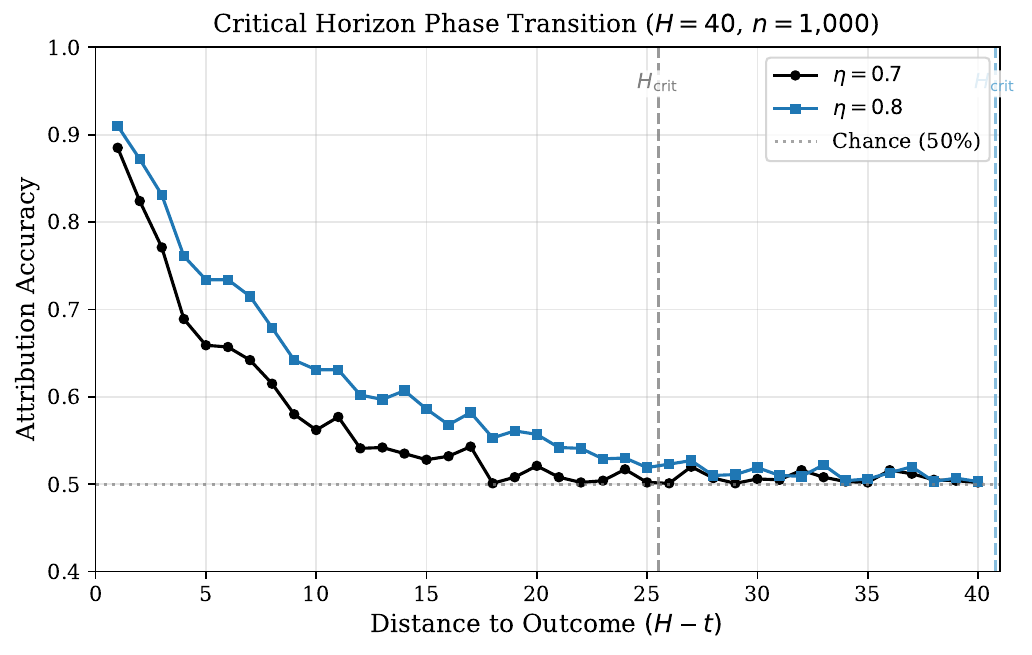}}
{Validation of the Critical Horizon (Corollary~\ref{cor:horizon}). Attribution accuracy drops to chance (50\%) at or before the predicted $H_{\mathrm{crit}}$ (dashed lines) for both contraction rates. For $\eta = 0.7$, $H_{\mathrm{crit}} \approx 26$; for $\eta = 0.8$, $H_{\mathrm{crit}} \approx 41$, which exceeds the horizon $H = 40$, so accuracy remains above chance throughout.\label{fig:critical_horizon}}
{}
\end{figure}

\subsection{Objective Mismatch}
\label{subsec:num-mismatch}

Definition~\ref{def:two_objectives} identifies a fundamental gap between step-quality ($J_{\mathrm{add}}$) and validity ($J_{\mathrm{mult}}$) objectives: a chain can have most steps correct yet still produce a wrong answer. We test whether this ``mostly correct fallacy'' occurs in practice by evaluating 50 GSM8K reasoning chains. We score each step using an automated heuristic (numerical consistency of intermediate calculations); because this heuristic is imperfect, the results below are \emph{illustrative} of the mismatch phenomenon rather than a definitive measurement of its magnitude. We compare additive rewards (fraction of steps rated correct) against a binary validity indicator (1 if the final answer is correct, 0 otherwise).

Figure~\ref{fig:objective_mismatch}(a) plots each chain's additive reward (x-axis) against its validity indicator (y-axis: 1 for correct, 0 for wrong). Dots along the top are chains that produce a correct final answer; dots along the bottom are chains with a wrong final answer. The shaded region marks the ``mostly correct but wrong'' zone: additive reward $\geq 0.80$ yet validity~$= 0$. Eight chains (16\%) fall in this zone. Six of these have \emph{every} step rated correct (additive reward~$= 1.0$) but still produce a wrong final answer, because the step-level heuristic evaluates numerical consistency of intermediate calculations but cannot catch errors in problem setup or final answer extraction. Panel~(b) shows the distribution of additive rewards separately for correct-answer and wrong-answer chains; the dashed line marks the $0.80$ threshold. Both correct-answer and wrong-answer chains concentrate at high additive reward, confirming that step-level scores do not reliably separate valid from invalid reasoning. These chains span 23--48 reasoning steps, so the phenomenon is not an artifact of short chains. An additive objective would assign near-maximal reward to all eight, reinforcing the reasoning patterns that produce wrong answers; the validity indicator correctly identifies them as failures.

\begin{figure}
\FIGURE
{\includegraphics[width=0.95\textwidth]{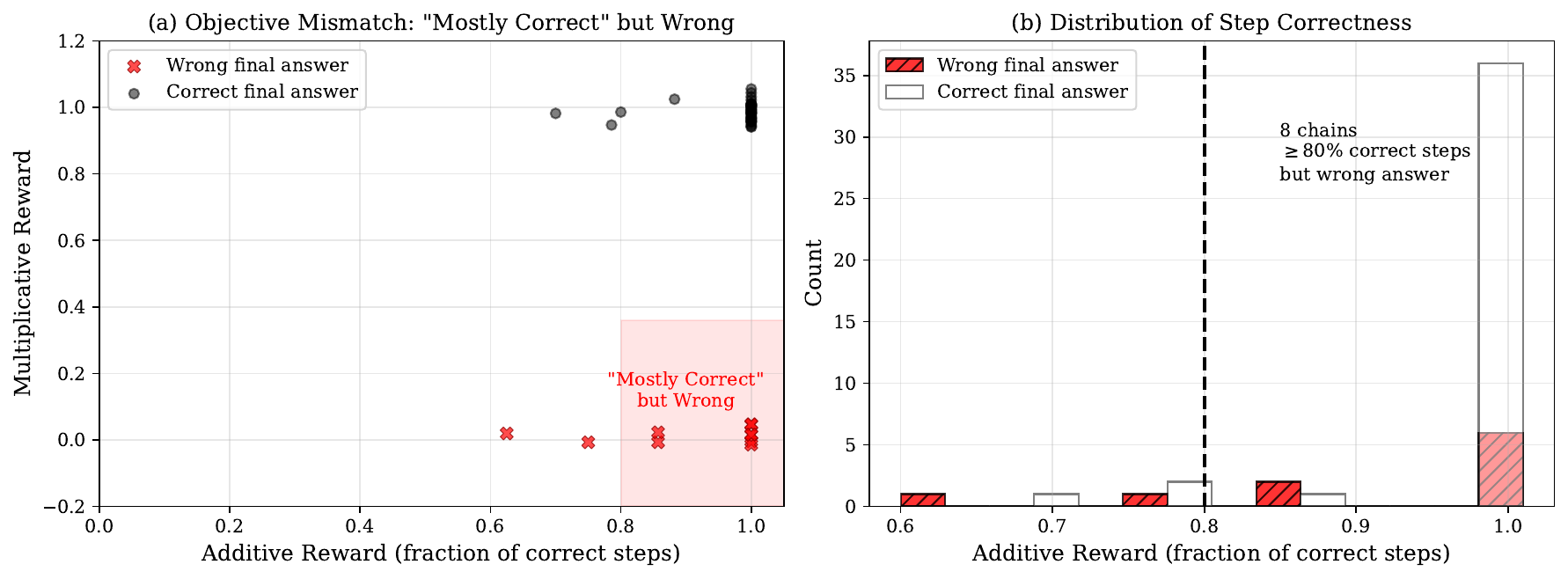}}
{Validation of the Objective Mismatch (Definition~\ref{def:two_objectives}). (a)~Additive vs.\ multiplicative reward for 50 GSM8K chains. The shaded region highlights 8 chains (16\%) that are ``mostly correct'' (high additive reward) but wrong (multiplicative reward~$= 0$). (b)~Distribution of step correctness, showing that wrong-answer chains can have nearly all steps correct.\label{fig:objective_mismatch}}
{}
\end{figure}

\section{Case Study: Supervision Design for Language Model Reasoning}
\label{sec:llm-application}

We now apply the general framework to a domain of significant current interest: training large language models to perform multi-step reasoning. This case study shows how the information-theoretic limits translate into concrete supervision design guidance for AI practitioners and clarifies the connection to intermediate inspection principles (Section~\ref{sec:inspection-design}).

Large language models present a puzzling credit assignment problem: token-level dynamics are deterministic (appending a token to a prefix yields a unique result), yet outcome supervision fails empirically. Value networks trained on terminal rewards produce near-constant outputs for early reasoning steps \citep{vineppo2024}. Our framework resolves this paradox through the state abstraction of Section~\ref{subsec:abstraction}.

Supervision targets semantic content, not surface representation: many distinct token sequences can encode the same meaning, so the abstraction map $\phi$ collapses token-level trajectories into a much smaller semantic state space. Writing $Z_t=\phi(S_t)$ for the semantic content of a partial reasoning chain, different token-level choices can lead to the same $Z_{t+1}$, inducing mixing in the \emph{abstract} Markov kernel even when the underlying token dynamics are deterministic. This many-to-one abstraction is the source of contraction, and it is precisely why outcome-only supervision can fail even when next-token prediction is deterministic. The lumpability condition (Definition~\ref{def:abstraction}) becomes the requirement that semantic transitions depend only on current meaning, not on which specific token sequence expresses it.

\begin{example}[Proof State Abstraction]
\label{ex:proof-llm}
The observable state $S_t$ is the complete text of a partial proof. The semantic state $Z_t=\phi(S_t)$ captures which statements have been established, what remains to prove, and the current proof strategy. Different phrasings of equivalent proofs map to the same semantic state. The semantic space $\mathcal{Z}$ is much smaller than the space of all token sequences, but it is the relevant space for evaluating correctness.
\end{example}

The supervision design question in this section is: \emph{given a target reasoning depth, what feedback should be provided (and where) to keep early-step attribution within the critical horizon?} We develop the answer in a logical sequence: (i) relate exploration mechanisms (temperature) to contraction and hence critical horizon; (ii) derive the prescription that local feedback (process supervision / verification) plays the same role as inspections in Section~\ref{sec:inspection-design}; and (iii) interpret width-based methods through the lens of variance reduction and correlation limits.

\subsection{Temperature Controls Contraction and the Critical Horizon}
\label{subsec:llm-temperature}

For language models, a primary control knob for contraction is sampling temperature. A softmax policy with temperature $\tau$ selects actions according to $\pi_\tau(a \mid z) \propto \exp(\ell(z,a)/\tau)$, where $\ell(z,a)$ denotes the logits. Low temperature concentrates probability on high-logit actions; high temperature spreads probability toward uniform. The policy induces a semantic transition kernel
\[{
K_\tau(z' \mid z) \;=\; \sum_{a \in \mathcal{A}} P(z' \mid z,a)\,\pi_\tau(a\mid z),
}\]

where $P(z'\mid z,a)$ is the probability of reaching semantic state $z'$ when taking token-level action $a$ from semantic state $z$. Even if token-level transitions are deterministic, the semantic transition is stochastic because $\phi$ collapses many token sequences into the same semantic state.

\begin{proposition}[Temperature Extremes and $\chi^2$-Contraction]
\label{prop:softmax}
Let $\{K_\tau\}_{\tau>0}$ be the semantic transition kernels induced by softmax sampling at temperature $\tau$.
\begin{enumerate}
\item (\emph{High-temperature mixing.}) If, as $\tau\to\infty$, $K_\tau(\cdot\mid z)\to \nu(\cdot)$ for all $z$ for some distribution $\nu$ independent of $z$, then $\eta_{\chi^2}(K_\infty)=0$.
\item (\emph{Low-temperature determinism.}) If, as $\tau\to 0$, $K_\tau(z'\mid z)\to \mathbf{1}[z'=f(z)]$ for a bijection $f:\mathcal{Z}\to\mathcal{Z}$, then $\eta_{\chi^2}(K_0)=1$.
\end{enumerate}
\end{proposition}

\begin{proof}{Proof.}
For Item~(1), if $K_\tau(\cdot\mid z)\to \nu(\cdot)$ for all $z$, then $PK_\infty=QK_\infty=\nu$ for any input distributions $P,Q$, hence $\chi^2(PK_\infty\|QK_\infty)=0$ and $\eta_{\chi^2}(K_\infty)=0$.
For Item~(2), if $K_0(z'\mid z)=\mathbf{1}[z'=f(z)]$ for a bijection $f$, then $K_0$ is a permutation kernel, so $\chi^2(PK_0\|QK_0)=\chi^2(P\|Q)$ and $\eta_{\chi^2}(K_0)=1$.
\Halmos
\end{proof}

Proposition~\ref{prop:softmax} clarifies the two extreme regimes. High temperature drives semantic transitions toward state-independent mixing ($\eta_{\chi^2} = 0$, maximal information loss), while low temperature drives them toward a deterministic permutation ($\eta_{\chi^2} = 1$, no information loss). Between these extremes, higher temperature increases action diversity and can raise the minimum transition probability $p_{\min}(\tau) := \min_{z,z'} K_\tau(z' \mid z)$. When this occurs, Proposition~\ref{prop:diversity} yields the tighter bound $\eta_{\chi^2}(K_\tau) \leq 1 - |\mathcal{Z}| \cdot p_{\min}(\tau)$, so contraction strengthens and signal decay accelerates with temperature.

The implication is not that one should always lower temperature. Rather, when exploration is necessary and contraction is non-negligible, outcome-only supervision cannot reliably attribute errors to early reasoning steps beyond the critical horizon. In that regime, supervision must be intermediate: provide process labels or verification signals at internal points of the chain so that each early step lies within the attenuation distance of an informative checkpoint, exactly as in the inspection design problem of Section~\ref{sec:inspection-design}.

\begin{example}[Contraction in a Reasoning Model]
\label{ex:eta_computation_llm}
Consider a semantic space with three states representing reasoning validity: ``on track,'' ``minor error,'' and ``fatal error.'' Suppose the policy induces the transition kernel
\[
K = \begin{pmatrix}
0.7 & 0.2 & 0.1 \\
0.3 & 0.4 & 0.3 \\
0.1 & 0.2 & 0.7
\end{pmatrix}.
\]
The minimum entry is $p_{\min}=0.1$, so Proposition~\ref{prop:diversity} gives $\eta \le 1 - |\mathcal{Z}|\times 0.1 = 1 - 3\times 0.1 = 0.7$. With $\eta=0.7$ and $n=10^6$ samples at $\Delta^2=0.1$, the critical horizon is approximately 32 steps.
\end{example}

When the semantic kernel is unknown, practitioners can estimate $\eta$ by sampling continuations from varied prefixes and measuring distributional overlap (yielding an upper bound via Proposition~\ref{prop:diversity}), or by computing the empirical correlation between value estimates at consecutive steps as a proxy for contraction strength.

\subsection{Process Supervision and Verification as Intermediate Inspections}
\label{subsec:llm-prm}

The inspection design framework of Section~\ref{sec:inspection-design} translates directly into supervision design for LLM reasoning. A process reward model that evaluates the semantic correctness of reasoning step $u$ produces an observation $Y_u = g_u(Z_u)$, exactly the inspection output of Definition~\ref{def:inspection-observation}: the evaluator $g_u$ is the human annotator or automated verifier applied to the semantic state. Human-annotated labels \citep{lightman2023lets} and automated verifiers (code execution, proof checkers, tool use) differ in cost but play the same structural role: each resets the signal decay clock so that nearby steps lie within the critical horizon of an informative checkpoint.

With this identification, the per-segment feasibility condition (Lemma~\ref{lem:segment-feasible}) governs how densely process labels must be spaced: consecutive labels at reasoning steps $t_i$ and $t_{i+1}$ support reliable credit assignment for all intervening steps if and only if the cumulative semantic contraction satisfies $W([t_i, t_{i+1})) \leq \Gamma(n, \Delta^2, \epsilon)$. Steps whose distance to the nearest label exceeds $H_{\mathrm{crit}}(\epsilon)$ require exponentially many samples to evaluate reliably; for such steps, process labels are not merely helpful but necessary.

This mapping also clarifies the cost structure. Human annotation provides high-quality labels but at high per-label cost $c_{\mathrm{insp}}$. Systems such as Math-Shepherd \citep{wang2024math} reduce this cost by generating process labels via Monte Carlo rollouts from intermediate states, substituting computation for human effort. The resulting tradeoff is precisely the joint sample-inspection budget problem of Proposition~\ref{prop:budget-lb}: given a total budget $B = n(c_{\mathrm{out}} + m \cdot c_{\mathrm{insp}})$, the designer chooses the number of labeled checkpoints $m$ and the number of training trajectories $n$ to minimize worst-case attribution error.

\subsection{Width-Based Methods as Variance Reduction Within the Horizon}
\label{subsec:llm-width}

Several prominent LLM training methods rely on width: GRPO (Group Relative Policy Optimization; \citealp{deepseek2024}) samples $W$ rollouts per prompt and compares outcomes to compute policy gradients; best-of-$W$ strategies generate $W$ completions and select the highest-scoring one. In the language of Section~\ref{sec:width}, the prompt serves as the starting state $s$, each completion is a rollout with binary outcome $R_j$, and the empirical success rate $\hat{V}_W(s) = W^{-1}\sum_j R_j$ is the multi-rollout estimator of Definition~\ref{def:multirollout}. These methods therefore inherit both the variance reduction of Proposition~\ref{prop:width} and the correlation cap of Theorem~\ref{thm:correlated}: effective width saturates at $W_{\mathrm{eff}} \leq 1/\rho$, where $\rho$ reflects shared reasoning patterns activated by the prompt and encoded in the pretrained weights.

Applying Corollary~\ref{cor:width-horizon} with LLM parameters reveals a tension with the temperature analysis of Section~\ref{subsec:llm-temperature}. Higher temperature lowers $\rho$ by diversifying continuations, which increases $W_{\mathrm{eff}}$ and extends the critical horizon $H_{\mathrm{crit}} = \ln(n \cdot W_{\mathrm{eff}} \cdot \Delta^2)/\ln(1/\eta)$. But Section~\ref{subsec:llm-temperature} showed that higher temperature also strengthens semantic contraction (smaller $\eta$), which increases $\ln(1/\eta)$ in the denominator and \emph{shrinks} $H_{\mathrm{crit}}$. Width-based methods therefore face diminishing returns at both extremes of the temperature dial: low temperature yields high correlation ($W_{\mathrm{eff}} \approx 1$), while high temperature yields strong contraction. The practical operating regime is an intermediate temperature where both effects are moderate.

The preceding analysis assumes that errors are diverse: different rollouts make different mistakes, so averaging reduces noise. When rollouts share systematic error patterns, the effective correlation $\rho$ increases and width-based methods become less effective, reinforcing the need for process supervision in deep reasoning chains.

\end{APPENDICES}

\end{document}